\documentclass[english]{article}

\usepackage[preprint,nonatbib]{neurips_2023}

\usepackage{wrapfig}
\usepackage[T1]{fontenc}
\usepackage[latin9]{inputenc}
\usepackage{amsmath}
\usepackage{amsthm}
\usepackage{amssymb}
\usepackage[colorlinks]{hyperref}
\usepackage[capitalize,noabbrev]{cleveref}
\usepackage{xcolor}
\usepackage{bm}
\usepackage{multirow}
\usepackage{booktabs} 

\usepackage{tcolorbox}

\usepackage{graphicx}

\usepackage{subcaption}

\usepackage{array}

\usepackage{setspace}
\definecolor{niceblue}{rgb}{0.10, 0.14, 0.76} 

\definecolor{nicered}{rgb}{0.70, 0.0, 0.0}

\AtBeginDocument{%
\hypersetup{
    citecolor=niceblue,
    linkcolor=red,   
    urlcolor=niceblue,
    linktoc=none}
}

\usepackage[numbers]{natbib}  
\usepackage{natbib}
\newcommand{\score}{\text{LES}}
\newcommand{\Score}{\text{Latent Exploration Score}}
\newcommand{\noreg}{\text{LSO (GA)}}
\newcommand{\likelihood}{\text{Likelihood}}

\newcommand{\bound}{\text{LSO (L-BFGS)}}
\newcommand{\turbo}{\text{TuRBO}}
\newcommand{\uc}{\text{UC}}
\newcommand{\prior}{\text{Prior}}
\newcommand{\zale}{\text{Zaleplon MPO}}
\newcommand{\pdop}{\text{Perindopril MPO}}
\newcommand{\rano}{\text{Ranolazine MPO}}
\newcommand{\rdfilter}{\texttt{rd\_filters}}

\newcommand{\change}[1]{\textcolor{black}{#1}}
\newcommand{\oldchange}[1]{\textcolor{black}{#1}}
\newcommand{\bz}{\boldsymbol{z}}
\newcommand{\bx}{\boldsymbol{x}}
\newcommand{\bA}{\boldsymbol{A}}
\newcommand{\bb}{\boldsymbol{b}}
\newcommand{\bp}{\boldsymbol{p}}

\RequirePackage{algorithm}
\RequirePackage{algorithmic}

\renewcommand{\P}{\mathbb{P}}
\newcommand{\E}{\mathbb{E}}

\newcommand{\dec}{\mathbf{G}_\theta}
\newcommand{\declog}{\mathbf{L}_\theta}

\newcommand{\enc}{\mathbf{E}_\theta}
\newcommand{\nn}{f_\theta}

\newcommand{\seqlen}{\textup{L}}
\newcommand{\sur}{\hat{f}}
\newcommand{\acq}{\mathcal{A}}
\newcommand{\obj}{\mathcal{M}}
\newcommand{\vocsize}{\textup{D}}
\newcommand{\validset}{\mathcal{V}}
\newcommand{\blackbox}{\mathcal{M}}

\newcommand{\scoref}{\mc{S}}
\newcommand{\scorefrho}{\mc{S}_\rho}
\newcommand{\logp}{\text{logP} }

\newcommand{\R}{\mathbb{R}}

\newcommand{\mc}[1]{\mathcal{#1}}

\newcommand{\bs}[1]{\boldsymbol{#1}}

\makeatletter
\theoremstyle{plain}
\newtheorem{theorem}{Theorem}[section]

\newtheorem{lemma}[theorem]{Lemma}

\theoremstyle{definition}

\newtheorem{example}[theorem]{Example}

\theoremstyle{remark}
\newtheorem{remark}[theorem]{Remark}

\title{Mitigating over-exploration in latent space optimization using $\score$}

\makeatother

\usepackage{babel}
\author{%
  Omer Ronen\textsuperscript{1} \quad
  Ahmed Imtiaz Humayun\textsuperscript{2} \quad
  Richard Baraniuk\textsuperscript{2} \quad
  Randall Balestriero\textsuperscript{3} \quad
  Bin Yu\textsuperscript{1} \\
  \texttt{omer\_ronen@berkeley.edu} \quad
  \texttt{imtiaz@rice.edu} \quad
  \texttt{richb@rice.edu} \\
  \texttt{rbalestr@brown.edu} \quad
  \texttt{binyu@berkeley.edu} \\
  \vspace{0pt} \\
  \textsuperscript{1}Department of Statistics, UC Berkeley \\
  \textsuperscript{2}Department of Electrical and Computer Engineering, Rice University \\
  \textsuperscript{3}Department of Computer Science, Brown University
}

\begin{document}

\maketitle

\begin{abstract}
We develop $\Score$ ($\score$) to mitigate over-exploration in Latent Space Optimization (LSO), a popular method for solving black-box discrete optimization problems. 
LSO utilizes continuous optimization within the latent space of a Variational Autoencoder (VAE) and is known to be susceptible to over-exploration, which manifests in unrealistic solutions that reduce its practicality. 
$\score$  leverages the trained decoder's approximation of the data distribution, and can be employed with any VAE decoder--including pretrained ones--without additional training, architectural changes or access to the training data. Our evaluation across five LSO benchmark tasks and twenty-two VAE models demonstrates that $\score$ always enhances the quality of the solutions while maintaining high objective values, leading to improvements over existing solutions in most cases.
We believe that new avenues to LSO will be opened by $\score$' ability to identify out of distribution areas, differentiability, and computational tractability. Open source code for $\score$ is available at \href{https://github.com/OmerRonen/les}{\texttt{https://github.com/OmerRonen/les}}.
\end{abstract}

\section{Introduction}\label{sec:intro}
\begin{figure}[ht]\label{fig:teaser}
    \centering
\includegraphics[width=1\linewidth]{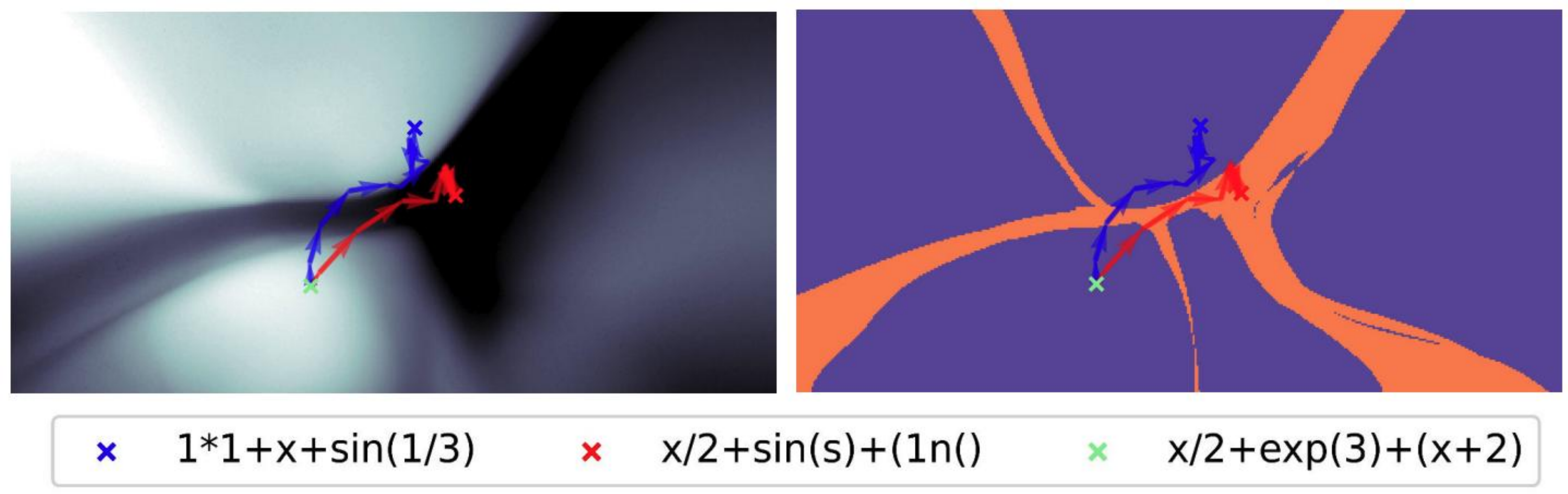}
    \caption{Incorporating $\score$ promotes valid solutions. We consider the task of approximating the expression \texttt{1/3 + x + sin(x * x)}, using LSO. Optimization trajectories with (blue) and without (red) $\score$ constraint in the latent space of a VAE are projected onto a two-dimensional subspace that contains the starting point and the end-points obtained after 10 gradient ascent steps. In the left panel, we show the $\score$ score for latent vectors on the two-dimensional subspace, with darker shades corresponding to lower $\score$. In the right panel, we show the validity of the decoder outputs for each latent vector, with orange denoting invalid generations. High $\score$ values correlate with valid areas, and incorporating $\score$ in LSO produce an expression that adheres to the grammatical rules of \cref{ex:expressions}.}
\end{figure}

Many important tasks in scientific fields, such as small molecule discovery and protein engineering, can be framed as discrete black-box optimization problems. In contrast to conditional sampling-based approaches, including GFlowNet \cite{bengio2023gflownet} and Diffusion \cite{corso2022diffdock, igashov2024equivariant}, which are better suited for applications like linker design \cite{du2024machine}, optimization is particularly effective when the goal is to improve a specific property, such as enhancing a drug's safety.

LSO was recently developed to enhance the sample efficiency of discrete optimization algorithms, such as genetic algorithms, in the black-box setting \cite{gomez2018automatic}. LSO transfers the optimization problem to the domain of the latent space of a VAE,
 which can be efficiently explored using continuous optimization methods. However, ensuring that LSO solutions respect the structure of the original space remains a challenge. To illustrate this issue, we  provide some examples.

\begin{example}[Arithmetic expressions]\label{ex:expressions}
    An expression built up using numbers, arithmetic operators and parentheses is called an arithmetic expression. However, not every sequence of the above elements correspond to a valid expression. For instance, the expression "\texttt{sin(x) + x}" is a valid expression while "\texttt{sin(xxx}" is not.
\end{example}   

\begin{example}[Simplified molecular-input line-entry system (SMILES)]\label{ex:smiles}
    SMILES provides a syntax to describe molecules using short ASCII strings. Atoms are represented by letters (e.g., water:"O"), bonds are represented by symbols (e.g., triple: "\#", double: "="), branches are represented in parentheses and cyclic structures are represented by inserting numbers at the beginning and the end.  
    Like the arithmetic expressions case, not every combination of the elements described above corresponds to a valid molecule. For example, while "\texttt{C1CCCCC1}" is valid, both  "\texttt{C1CCCCC2}" and "\texttt{C1CCCCC)}" are not.

  \end{example}  
  
\begin{example}[Quality filters for molecules]\label{ex:filters}
Chemists seek molecules that not only optimize desired chemical properties but are also stable and easy to synthesize. This has led to the development of rules such as Lipinski's Rule of Five (RO5 \cite{lipinski1997experimental}), which helps determine if the bioavailability (i.e., the proportion of a drug or other substance that enters the circulation when introduced into the body) of a given compound meets a certain threshold. For example, RO5 suggests that poor absorption is more likely when the octanol-water partition coefficient ($\logp$) exceeds 5. Similarly, the Pan Assay Interference Compounds \cite{baell2010new} filter helps in identifying false positives in assay screenings. The $\rdfilter$ \cite{Walters2019} package has curated many such rules and is considered a "high precision, low recall surrogate measure" \cite{brown2019guacamol}. Following \cite{NEURIPS2021_06fe1c23} we consider a sample valid if it passes the $\rdfilter$ filters \footnote{We use the default Inpharmatica rule set comprised of 91 alerts}

\end{example}  

Numerous directions have been explored to overcome the challenge of providing valid solutions, including specialized VAE architectures \cite{kusner2017grammar,jin2018junction} or robust representations for discrete data \cite{krenn2020self}. Additionally, constrained objectives can be formulated under the assumption that one has access to a function which quantifies the validity of any point in the latent space \cite{griffiths2020constrained}. However, in many realistic scenarios, such as \cref{ex:filters}, these solutions may not be directly applicable, as the structure of the sequence space may not be sufficiently well understood. To address this, \cite{NEURIPS2021_06fe1c23} proposed using an estimator of the uncertainty of the decoder,  based on the variational approximation to a posterior distribution over the VAE parameters, encouraging LSO to respect the sequence space structure. Although this approach proved effective, the non-differentiable nature of the uncertainty score required its integration into LSO through heuristic approaches. Additionally, the computation of the uncertainty score is not exact (i.e., it relies on variational approximation and Monte Carlo sampling) and requires significant amount of time to compute. Therefore, there is a need for robust methods that work across different VAEs and sequence space structures, and can be easily integrated into existing LSO pipelines.

To achieve this goal, we develop $\score$, a score that can be used as a constraint in LSO optimization to increase the number of solutions that respect a given structure. The distinctive characteristics of $\score$ are differentiability and robustness that allow its easy integration into existing LSO pipelines. Specifically, our contributions are as follows:
\begin{itemize}
    \item We introduce $\score$, a score that achieves higher values in regions of the latent space closer to the training data. Our results demonstrate that $\score$ is highly effective at identifying regions that preserve the structure of the sequence space. Although $\score$' computation scales cubically with the latent dimension, it is up to 80\% faster than the current state-of-the-art for identifying out-of-distribution data points in the latent space of generative models for discrete sequences (\cref{tab:roc_scores,tab:wc}).
    \item We develop a numerically stable optimization procedure to incorporate $\score$ as a constraint in LSO.
    \item We evaluate $\score$-constrained LSO across thirty optimization tasks, including twenty-two VAEs and five benchmark problems, demonstrating its robustness in generating valid solutions and achieving high objective values. Specifically, in 19 out of the 30 LSO experiments, our method either finds the best solution on average or achieves a solution within 1 standard deviation of the best solution across 10 independent runs. This outperforms the six alternative methods we considered by 19\% (\cref{tab:top1,tab:top20}).
\end{itemize}

\section{Background: Latent Space Optimization}\label{sec:lso}
LSO is a method for solving black box optimization problems in discrete and structured spaces, \oldchange{such as the space of valid arithmetic expressions}. Formally, let $\validset\subset \R^{\seqlen\times\vocsize}$ be a discrete and structured space, represented as a sequence of $\seqlen$ one-hot vectors of dimension $\vocsize$. We represent sequences of length $\seqlen$ of categorical variables with $\vocsize$ categories. $\seqlen$ is set as the maximum sequence length that we are optimizing for, and one of the $\vocsize$ categories is used as an "empty" category. For instance, in the case of valid arithmetic expressions, $\validset$ would be the set of sequences that define such expressions. Let $\blackbox:\validset\rightarrow\R$ be the objective function. LSO solves,
\begin{align}\label{eq:lso}
\arg\max_{x\in\validset}\blackbox(x).
\end{align}

In this setting, we assume that evaluations of the objective function ($\blackbox$) are expensive to conduct. For example, the objective may be the binding affinity of a compound to a given protein, measured through a wet lab experiment.  

A popular approach to solve \cref{eq:lso} is Bayesian Optimization (BO), which utilizes first order optimization of a surrogate model for $\blackbox$. However, since the space is discrete, first order optimization cannot be directly applied.  \newline
In an attempt to make BO applicable for solving \cref{eq:lso}, \cite{gomez2018automatic} proposed to transfer the optimization problem into that over a domain of the latent space of a deep generative model and subsequently perform BO in this space. The main idea is to (1) learn a continuous representation of the discrete objects (e.g., using a VAE) and (2) perform BO in the latent space while decoding the solution at each step. Formally, given a pre-trained encoder ($\enc$) and decoder ($\dec$) the initial labelled dataset $\mc{D} = \{\bx_i, y_i\}_{i=1}^n$ is first encoded into the latent space $\mc{D}^{\bz} = \{\bz_i = \enc(\bx_i), y_i\}_{i=1}^n$. Using the encoded dataset, a BO procedure is conducted, which we describe in \cref{alg:lso}. Most commonly, a Gaussian process is used as the surrogate model for $\blackbox$, and the acquisition function is the expected improvement \citep{frazier2018bayesian}, defined as:

\begin{align}\label{eq:ei}
     \acq_{\sur}(\bz)= \E_{\sur} \max(\sur(\bz) - \max_{i} y_i, 0)
\end{align}

where the expectation is with respect to the distribution of the function $\sur$, conditioned on $\mc{D}^{\bz}$.
\begin{algorithm}[H]
    \caption{Latent Space Optimization}
    \label{alg:lso}
    \textbf{for} $t=1$ \textbf{to} $T$ \textbf{do}
    \begin{enumerate}
            \item Fit a surrogate model $\sur$ to the encoded dataset, $\mc{D}^{\bz}$
    \item Generate a new batch of query points by optimizing a chosen acquisition function ($\acq$)
    \begin{align}\label{eq:acq_opt}
        \bz^{(\text{new})} = \arg\max_{\bz} \acq_{\sur}(\bz)
    \end{align}
    \item Decode $\bx^{(\text{new})} = \dec(\bz^{(\text{new})})$, evaluate the corresponding true objective values ($y^{(\text{new})} = \obj(\bx^{(\text{new})})$) and update $\mc{D}^{\bz}$ with ($\bz^{(\text{new})}, y^{(\text{new})}$).
    \end{enumerate}
\end{algorithm}

\paragraph{Over-exploration in LSO} 
Multiple studies \citep{NEURIPS2021_06fe1c23,kusner2017grammar} have found that unconstrained latent space optimization (LSO) often yields solutions that disregard the aforementioned structures. For example, when searching for arithmetic expressions, invalid equations like "$ssin(xxx$" frequently occur. Similarly, many solutions in molecule searches fail to pass basic quality filters (\cref{ex:filters}), limiting their practical utility \citep{maus2022local}.  

While acquisition functions such as expected improvement (\cref{eq:ei}) are designed to balance exploration and exploitation based on the estimated uncertainty from the Gaussian process model for $\obj$. The frequent generation of invalid solutions during acquisition optimization, which implies that the estimated uncertainty can be problematic in this setting, underscores the need for additional regularization \citep{tripp2020sample}, which we aim to address.

To mitigate over-exploration, we propose adding a penalty to \cref{eq:acq_opt}. The penalty uses a new score, giving higher values over the latent space valid set, defined as:
    \begin{align}\label{def:val_reg}
        \{\bz;\dec(\bz)\in\validset\},
    \end{align}
where $\dec:\mathcal{Z}\rightarrow\R^{\seqlen\times\vocsize}$ is the decoder network, and  $\validset\subset\R^{\seqlen\times\vocsize}$ is the set of valid sequences.

The derivation of our score leverages the Continuous Piecewise Affine (CPA) representation of neural networks, which we briefly review below.

\paragraph{Deep generative networks as CPA} Following \citep{humayun2022polarity, humayun2021magnet, balestriero2018mad, balestriero2024geometry}, we consider the representation of Deep Generative Networks (DGNs) as Continues Piecewise Affine (CPA) Splines operators. Let $\nn$ be any neural network with affine layers and piecewise affine activations then it holds that
\begin{align}\label{eq:nn_cpas}
    \nn (\bz) = \sum_{\omega\in\Omega} \left(\bA_\omega \bz + \bb_\omega\right)1_{\{\bz\in\omega\}},
\end{align}
where $\Omega$ is the input space partition induced by $\nn$, $\omega$ is a particular region and the parameters $\bA_\omega$ and  $\bb_\omega$ defines the affine transformation depending on $\omega$. 

In cases where the neural network $\nn$ is not composed solely of piecewise affine layers and activations, we leverage the result from \cite{daubechies2022neural} to assert that \cref{eq:nn_cpas} either exactly represents $\nn$ or provides a sufficiently accurate approximation for our practical purposes \citep{humayun2022polarity}. We therefore argue that all the decoder neural networks included in our study (i.e., GRU, LSTM, and Transformers) can be approximated with high accuracy as continuous piecewise affine (CPA) functions.

\section{A Latent Exploration Score to Reduce Over-Exploration in LSO}\label{sec:score}

In this section, we introduce $\Score$ ($\score$), our new score to reduce over-exploration in LSO. We begin by motivating $\score$ and proceed to formally derive it in \cref{sec:score_der}. In \cref{sec:sc_val}, we provide empirical evidence that $\score$ gives higher values in the latent space valid set. The use of $\score$ to regularize LSO is left for \cref{sec:exprs}.

\paragraph{Motivation}

Our goal is to develop a meaningful constraint for optimizing the acquisition function within a latent space of a given VAE. Specifically, we aim to construct a constraint that is a continuous function of 
$\bz$ with higher values, indicating that it is more likely that $\bz$ resides within valid regions of the latent space (\cref{def:val_reg}).  

Such a score should be higher in regions near training data points, assuming most of VAE training data is valid. To achieve this, we treat the latent space of the VAE as a probability space, i.e. $\bz\sim p_{\bz}$, for some prior distribution $p$ (for example standard Gaussian). The prior should reflect our best guess for the distribution of the observed data in the latent space. Solutions are mapped back to sequences by the decoder through a deterministic (we do not consider $\bx$ to follow a conditional distribution given $\bz$) transformation of the latent vectors. Therefore, any distribution on the latent space defines a distribution over the space of sequences.  Our score uses the density function of the push-forward measure of $\bx = \dec(\bz)$, which we call the sequence density. Consequentially, our score depends only on the decoder network, not the encoder, and can potentially be applied to other generative models like GANs or diffusion models.
\paragraph{Why use the sequence density function?} We argue that for a well-trained decoder network, the density should be higher in areas of the sequence space close to the training data. To see why, consider a decoding model $\dec$ trained on a dataset $\{(\bz_{i}, \bx_i)\}_{i=1}^n$. The average loss (L) at $\bz$ is
\begin{align}
    \ell(\dec(\bz)) = \mathbb{E}_{\bx | 
\enc(\bx) = \bz} L(\dec(\bz), \bx).
\end{align}
 As the training process is designed to minimize the population loss:  $\mathbb{E}\ell(\dec(\bz))$, if successful, we hypothesize that the distribution of $\dec(\bz)$  puts higher weight in the areas where $\ell(\dec(\bz))$ is low. Since we expect most of the training data to be valid 
 and to achieve low expected loss, the sequence density should put higher weight on the latent space valid set. In, \cref{sec:sc_val} we provide an empirical validation for this hypothesis, for \cref{ex:expressions,ex:smiles,ex:filters}.  We highlight that this relationship between the valid set and the sequence density depends on how well the decoder fits the data.

 \subsection{Derivation of \score}\label{sec:score_der}

\begin{figure}[H]
    \centering
    \includegraphics[width=0.5\linewidth]{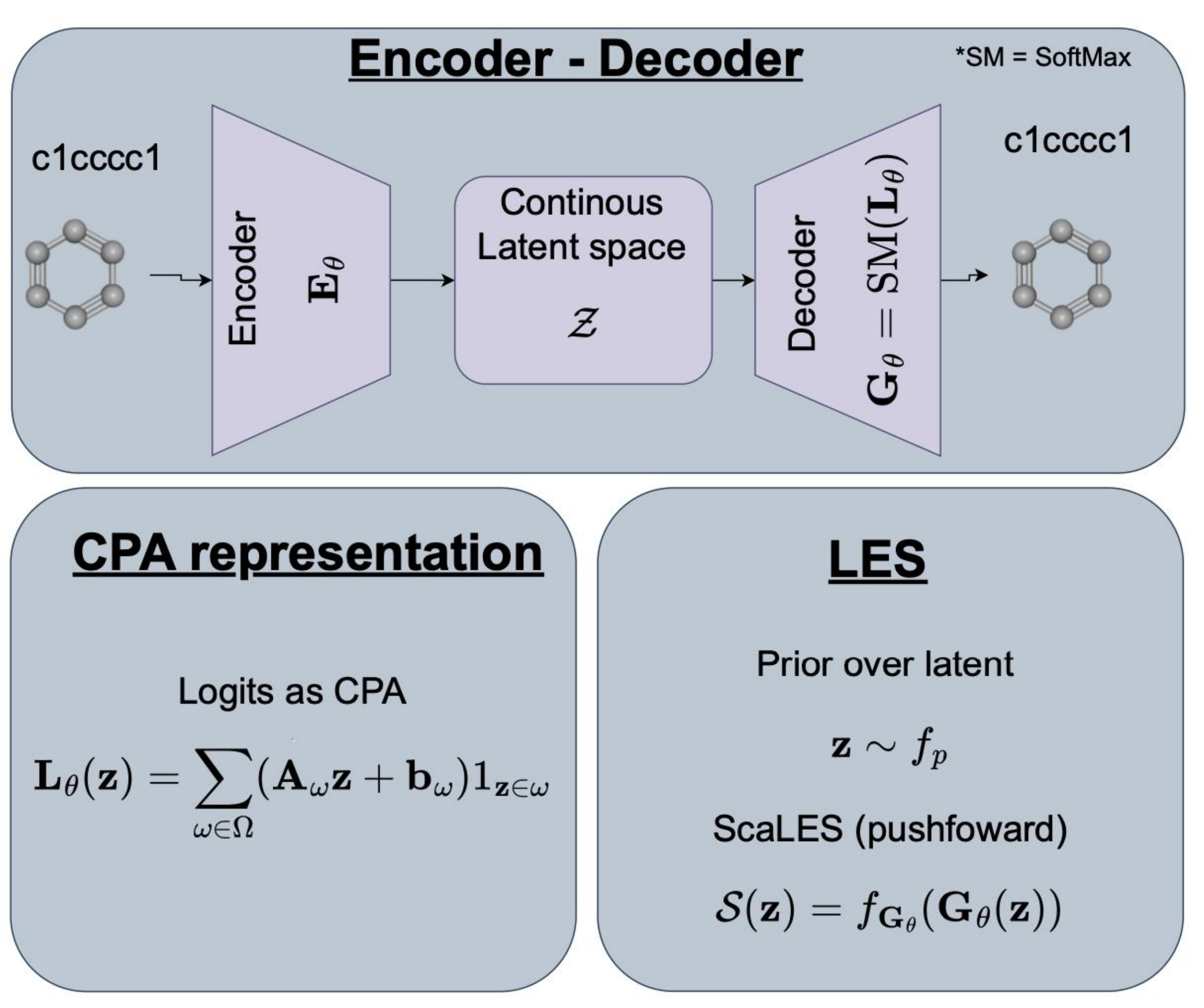}
    \caption{Derivation of $\score$. The decoder network ($\dec$), which maps from the latent space to the output space, is assumed to be the composition of a softmax operation over a continuous piecewise affine (CPA) spline  operator. $\score$ is the density of a random variable ($\bz$) in the latent space, under the decoder transformation. Calculating $\score$ only requires a pre-trained decoder.}
    \label{fig:enter-label}
\end{figure}

\paragraph{Analytical formula for $\score$} 
DGNs for discrete sequences typically output a matrix of logits, transformed into normalized scores by the softmax function:
\begin{align}
    \dec(\bz) = \text{Softmax}(\declog(\bz)).
\end{align}

$\declog(\bz)$ is a $\vocsize \times \seqlen$ logits matrix ($\vocsize$ - vocabulary size, $\seqlen$ - sequence length) and $\text{Softmax}$ is the softmax operation applied to every column of $\declog(\bz)$.  In order to avoid a violation of the assumption that 
$\dec$ is bijective, we extend the function's output to include the normalizing constant for each column. We find that parametrizing the output to include the inverse of the normalizing constant (\cref{eq:extended_dgn}), helps in avoiding numerical instabilities that are caused by the constant being potentially very large. With this formulation, we can now derive the sequence density function.

\begin{theorem}[DGN sequence density]\label{thm:img_den}
Let
\begin{align}\label{eq:extended_dgn}
     \dec(\bz) &=  \left(\bp^{(1)}_{\bz}, (c^{(1)}_{\bz})^{-1}, \dots, \bp^{(L)}_{\bz}, (c^{(L)}_{\bz})^{-1}\right)
     = \bx_{\bz}
\end{align}

where $\bp^{(i)}_{\bz} = \text{Softmax}(\declog(\bz))_{.i}$ and $c^{(i)}_{\bz} = \sum_{j=1}^D e^{\declog(\bz)_{ji}}$. Assume that $\declog$ is bijective and can be expressed as a CPA (\cref{eq:nn_cpas}), and that $\bz\sim p_{\bz}$, then the density function of $\dec(\bz)$ is given by:
\begin{align}\label{eq:density_dgn}
    f_{p}(\bs{z})\sqrt{\det\left(
    \sum_{i=1}^\seqlen  (\bs{A}^{\dagger}_{i})^T (\bs{B}_{i})^T    \bs{B}_{i}\bs{A}^{\dagger}_{i}\right)}
\end{align}

for
\begin{align}
   \bs{B}_{i} &=  \left(\text{diag}\left(\frac{1}{(\bs{p}^{(i)}_{\bs{z}})_{1}}, \dots, \frac{1}{(\bs{p}^{(i)}_{\bs{z}})_{\vocsize}}\right), -\bs{1}\frac{1}{c^{(i)}_{\bs{z}}}\right)^T\\
   \bs{A}^{\dagger}_{i} &= \left(\bs{A}^{(1)}_{\omega}, \dots, \bs{A}^{(L)}_{\omega}\right)^{\dagger}_{(i\cdot \vocsize) : (i+1 \cdot \vocsize).},
\end{align}

where $\left(\bs{A}^{(1)}_{\omega}, \dots, 
\bs{A}^{(L)}_{\omega}\right)^{\dagger}$ 
is the Moore Penrose inverse of $\left(\bs{A}^{(1)}_{\omega}, \dots, \bs{A}^{(L)}_{\omega}\right)$, and $f_p$ is the density function of $p_{\bz}$.

\end{theorem}

The proof is provided in \cref{sec:proofs}. We define $\score$ to be the logarithm of the determinant term,
\begin{align}\label{eq:score_rho}
    \scoref(\bz) &=  \log\left(\sqrt{\det\left(
    \sum_{i=1}^\seqlen  (\bs{A}^{\dagger}_{i})^T (\bs{B}_{i})^T    \bs{B}_{i}\bs{A}^{\dagger}_{i}\right)}\right),
\end{align}
as the contribution of the prior is negligible in magnitude in all the decoders we study.
\begin{remark}
$\score$ can be calculated directly without assuming the decoder logits follow a CPA function \citep{ben1999change}. However, using the expessions derived in \cref{eq:score_rho} has two computational benefits for calculating the derivative of $\score$.
First, using Jacobi's formula and observing that $\sum_{i=1}^\seqlen  (\bs{A}^{\dagger}_{i})^T (\bs{B}_{i})^T    \bs{B}_{i}\bs{A}^{\dagger}_{i}$ is a quadatic formula of the softmax probabilites, we can calculate the derivative of $\score$ in closed form. Second, by the CPA assumption, the matrices  $\bs{A}^{(i)}_{\omega}$ are a constant function of $\bz$ and therefore $\frac{\partial \bs{A}^{(i)}_{\omega}}{\partial \bz} = 0$. As a result, we avoid the need to calculate the hessian of the decoder when taking the derivative of $\score$.
\end{remark}

\paragraph{Limitations of \cref{thm:img_den}}
Our derivation relies on the decoder logits being (i) a CPA operator and (ii) bijective between the latent space and the generated manifold in the ambient space. We argue that (i) is not a restrictive assumption, as approximation theory has already demonstrated that any continuous model can be approximated by a CPA network. Therefore, one always recovers Theorem 5 even when using non-CPA models (see \cite{daubechies2022neural}).

On the other hand, (ii) is a stronger assumption that practitioners should be mindful of, as it would invalidate ScaLES as a meaningful metric for comparing different samples. For (ii) to be violated, i.e., for \cref{eq:cov} to be incorrect, the Lebesgue measure of the set $C_{\dec} = \{\bz ;|; \exists \bz^* ; \dec(\bz) = \dec(\bz^*)\}$ must be larger than 0 (see \cref{lem:mes_0} for proof). This would suggest some degeneracy in the decoder function, where large regions of the latent space map to the same output, resulting in a zero gradient of the decoder with respect to its input. However, we believe this is rare in practice. In our experiments, where we compute the gradient of the decoder with respect to the input to calculate $\score$, we did not encounter instances where the gradient was zero.

Although we do not formally validate (ii) (the bijectivity assumption), we argue through our empirical analysis in \cref{tab:roc_scores}, conducted across 22 VAEs (including pre-trained models), that (ii) may hold or, at the very least, serve as a reasonable approximation for real-world VAEs.  

Lastly, it is crucial to highlight that the ability of $\score$ to detect out-of-distribution data is closely tied to the decoder's capacity to accurately model the data. Although \cref{thm:img_den} holds for poorly trained decoders under the given conditions, we advise against relying on $\score$ in such cases.

\paragraph{Computing \score}

$\score$ is a function of the matrices $\bs{B}_i$ and $\bs{A}_\omega$.
The matrices $\bs{B}_i$ are a function of the logits and can be calculated using a single forward run of $\dec$. The matrix $\bs{A}_\omega$ is equal to the derivative of $\declog$ at $\bz$, and can therefore can be obtained using automatic differentiation (\citep{paszke2017automatic}). To  avoid the (pseudo) inversion of the matrix $\bs{A}_\omega$ we can exploit the inverse function theorem which states that $J\dec^{-1} = (J\dec)^{-1}$ and we can take $\scoref(\bz) = -0.5\log\det((A_\omega \frac{\partial \dec(\declog)}{\partial \declog})(A_\omega \frac{\partial \dec(\declog)}{\partial \declog})^T )$, where $\frac{\partial \dec(\declog)}{\partial \declog}$ (derivative of the softmax w.r.t the logits) admits a simple closed form solution.  
Ideally, $\score$ can be computed by performing all of the above calculations in parallel using a single forward call to the $\dec$ network. In addition, the computation of the determinant is done via SVD on a square matrix whose dimension is the latent dimension of the decoder ($d$) with complexity of $\mathcal{O}(d^3)$.

In \cref{tab:wc} we provide the wall clock times for calculating $\score$ for a batch of 20 latent vectors across all architectures and datasets studies in this work. $\score$ is compared with the Bayesian uncertainty score proposed by \cite{NEURIPS2021_06fe1c23} (with the default configuration: 10 sampled models and 40 sampled outcomes), which was previously used to regularize LSO. We also compare with a \textit{\likelihood} score:
\begin{align}\label{eq:lik}
    \ell(\bz) = \max_{\bx} p_{\dec}(\boldsymbol{X}=\bx|\boldsymbol{Z}=\bz)
\end{align}

where $p_{\dec}(\bz)$ is the distribution defined by the decoder softmax probabilities, which reflects the likelihood of the most likely $\bx$ conditioned on the latent vector $\bz$.

For 9 out of the 10 models, particularly when the latent dimension is larger (e.g., SMILES and SELFIES), $\score$ is computed faster than the Uncertainty score, achieving reductions of up to 85\% in some cases. As the $\likelihood$ score requires only a single forward pass of the decoder, it offers a more computationally efficient alternative, which comes with some performance trade-offs (\cref{sec:sc_val,sec:exprs}).

\begin{table}[ht]
\centering
\renewcommand{\arraystretch}{1.25} 
\caption{Wall clock times in seconds (lower is better) for calculating $\score$, the Bayesian uncertainty and the $\likelihood$ scores for a sample of 20 latent vectors on single A100 GPU.\\}\label{tab:wc}


\begin{tabular}{llcccc}
\toprule
Dataset & Arch. (dim.)  & \score & Uncertainty & Likelihood \\
\midrule

\multirow{3}{*}{Expressions} 
& GRU   (25) & 0.730 & 0.823 & 0.025 \\
& LSTM    (25) & 0.164 & 0.857 & 0.049 \\
& Transformer  (25) & 0.526 & 0.481 & 0.029 \\
\midrule

\multirow{3}{*}{SMILES} 
& GRU    (56) & 0.663 & 3.157 & 0.103 \\
& LSTM     (56) & 0.696 & 3.990 & 0.123 \\
& Transformer  (56) & 0.581 & 0.726 & 0.185 \\
\midrule

\multirow{4}{*}{SELFIES} 
& GRU         (75)  & 0.498 & 1.925 & 0.064 \\
& LSTM         (75)  & 0.525 & 2.442 & 0.071 \\
& Transformer   (75)  & 0.451 & 0.583 & 0.085 \\
& Transformer (256) & 7.422 & 42.882 & 0.410 \\
\midrule

\textbf{Average} & --  & 1.226 & 5.787 & 0.114 \\
\bottomrule
\end{tabular}

\end{table}

\subsection{Validating the Relationship Between LES and Valid Generation
}\label{sec:sc_val}
To assess \(\score\)'s ability to identify valid regions (as defined in \cref{ex:expressions,ex:smiles,ex:filters}), we sample data points in the latent space using the twenty-two VAEs studied in \cref{sec:exprs}. Specifically, we sample 500 data points from three distributions: train, prior ($\mathcal{N}(\bs{0}, \bs{I})$), and out-of-distribution ($\mathcal{N}(\bs{0}, \bs{I}\cdot 5)$). We decode each data point and determine if the decoded sequence is valid.

Identifying if a point in the latent space decodes into a valid sequence can be viewed as a classification problem, in which the different scores (i.e., $\score$ or the Bayesian uncertainty score) provide (unnormalized) probabilities for a sequence being valid. We measure the performance of these scores using the AUROC metric. Besides $\score$, the Bayesian uncertainty score and the $\likelihood$ score, we add three additional baseline scores for comparison. The first is the density of a standard Gaussian (\textbf{$\prior$}), which is the distribution the latent vectors are regularized to follow during VAE training. The second is the polarity score (\textbf{Polarity}) \citep{humayun2022polarity}, based only on the derivative of the decoder logits with respect to the latent vector, which shows the gains due to accounting for the softmax non-linearity in the derivation of $\score$ (\cref{thm:img_den}). We also consider the average distance to the closest three data points within a random sample of 1000 points from the training data in the latent space (\textbf{Train distances}).  

The results are shown in \cref{tab:roc_scores,tab:roc_scores_small}. $\score$ provides the best performance in 18 out of the 22 VAEs in this analysis, and in all cases provides a clear signal for identifying valid regions, as indicated by AUROC values that are at least 0.75. This is while being much faster to compute than the Uncertainty score (\cref{tab:wc}) and without the need to store a potentially large array of latent vectors.

\begin{table}[H]
\caption{AUROC values (higher is better) for identifying valid data points within the latent space, across datasets and decoder architectures. Data points are sampled from the training data, the VAE prior (standard Gaussian), and out-of-distribution data (Gaussian with std of 5). $\score$ achieves the best performance in most cases (18 out of 22) and on average.\\}\label{tab:roc_scores_small}
\centering
\small
\renewcommand{\arraystretch}{1.1} 

\setlength{\tabcolsep}{0.3em} 
{\begin{tabular}
{llclrrr}
\toprule
Dataset & Arch. & $\beta$ & $\score$\;  & Prior & Uncertainty & Likelihood \\
\midrule
\multirow{9}{*}{SMILES} 
& \multirow{3}{*}{GRU} 
    & 0.05 & 0.93 & 0.09 & 0.85 & \textbf{0.94} \\
& 
    & 0.1  & \textbf{0.94} & 0.12 & 0.84 & \textbf{0.94} \\
& 
    & 1.0  & 0.91 & 0.16 & 0.87 & \textbf{0.92} \\ \cline{2-7}
& \multirow{3}{*}{LSTM} 
    & 0.05 & \textbf{0.99} & 0.07 & \textbf{0.99} & 0.98 \\
& 
    & 0.1  & 0.89 & 0.21 & 0.89 & \textbf{0.90} \\
& 
    & 1.0  & \textbf{0.97} & 0.12 & 0.95 & 0.96 \\ \cline{2-7}
& \multirow{3}{*}{Transformer} 
    & 0.05 & \textbf{0.93} & 0.14 & 0.84 & 0.92 \\
& 
    & 0.1  & \textbf{0.94} & 0.14 & 0.87 & 0.93 \\
& 
    & 1.0  & \textbf{0.97} & 0.10 & 0.89 & 0.95 \\ \midrule

\multirow{9}{*}{Expressions} 
& \multirow{3}{*}{GRU} 
    & 0.05 & \textbf{0.96} & 0.38 & \textbf{0.96} & 0.88 \\
& 
    & 0.1  & \textbf{0.94} & 0.42 & \textbf{0.94} & 0.80 \\
& 
    & 1.0  & \textbf{0.94} & 0.57 & \textbf{0.94} & 0.89 \\ \cline{2-7}
& \multirow{3}{*}{LSTM} 
    & 0.05 & \textbf{0.96} & 0.38 & 0.90 & \textbf{0.96} \\
& 
    & 0.1  & \textbf{0.95} & 0.37 & 0.91 & \textbf{0.95} \\
& 
    & 1.0  & \textbf{0.95} & 0.56 & 0.91 & 0.91 \\ \cline{2-7}
& \multirow{3}{*}{Transformer} 
    & 0.05 & \textbf{0.91} & 0.43 & 0.86 & 0.90 \\
& 
    & 0.1  & \textbf{0.91} & 0.53 & 0.87 & 0.89 \\
& 
    & 1.0  & 0.86 & 0.70 & \textbf{0.92} & \textbf{0.92} \\ \midrule

\multirow{3}{*}{SELFIES} 
& \multirow{3}{*}{Transformer} 
    & 0.05 & \textbf{1.0} & 0.02 & 0.99 & 0.97 \\
& 
    & 0.1  & \textbf{0.99} & 0.03 & 0.96 & 0.98 \\
& 
    & 1.0  & \textbf{0.95} & 0.06 & 0.85 & 0.93 \\ \midrule

SELFIES~ (\cite{maus2022local})
& Transformer 
    & --   & \textbf{0.75} & 0.69 & 0.33 & 0.70 \\

\midrule
& Average  & & \textbf{0.93} & 0.29 & 0.88 & 0.91 \\
\bottomrule
\end{tabular}}

\end{table}

\section{LES-Constrained LSO}\label{sec:exprs}
In \cref{sec:sc_val} we showed that $\score$ is a robust score that obtains higher values in the latent space valid set (\cref{def:val_reg}). Furthermore, $\score$ is differentiable, which means it can easily be used to constrain any optimization problem. Therefore, we propose adding an explicit constraint to \cref{eq:acq_opt}, \oldchange{encouraging the solution to achieve a high $\score$ value}. We modify \cref{alg:lso} by penalizing step (2):

\begin{align}\label{eq:lso_constrained}
    \bz^{\text{new}} = \arg\max_{\bz} \acq(\sur(\bz)) + \textcolor{red}{\lambda\scoref(\bz)}.
\end{align}
\subsection{Experimental Setup}
\paragraph{VAE models}
To evaluate the effectiveness of $\score$ as a regularization method for LSO, we trained twenty-two VAEs, focusing on varying the decoder architectures and the $\beta$ parameter, which controls the trade-off between reconstruction loss and alignment with the prior (KL divergence term). All models use a convolutional encoder based on the architecture proposed by \cite{kusner2017grammar}, and were trained for 300 epochs using the Adam optimizer \cite{kingma2014adam} with a learning rate of 1e-3 and batch size of 256.

The VAEs for the Expressions dataset, sourced from \cite{kusner2017grammar}, were trained on 80k data points with a latent dimension of 25. The SMILES VAEs were trained on the ZINC250k dataset, consisting of approximately 250k drug-like molecules in SMILES format. Following \cite{kusner2017grammar} and \cite{NEURIPS2021_06fe1c23}, a latent dimension of 56 was used. For the SELFIES VAEs, we used a subset of approximately 200k molecules from the ZINC250k dataset that passed a set of quality filters  \cite{Walters2019}, using the SELFIES representation \cite{krenn2020self}, with a latent dimension of 75. Additionally, the pre-trained VAE by \cite{maus2022local} had a latent dimension of 256.

\paragraph{LSO setup}
\change{We begin by training a single-task Gaussian Process on an initial dataset. Each sample is mapped to a latent space and paired with its true objective value. For Expressions and SMILES VAEs, we use 500 data points. For SELFIES VAEs, we use 1500. Across all tasks, we generate 500 candidate solutions per problem, aligned with prior work \cite{kusner2017grammar,NEURIPS2021_06fe1c23} and reflective of real-world wet-lab constraints \cite{gao2022sample}. Solutions are proposed in batches of 20.}

For the SELFIES VAEs (both from \cite{maus2022local} and those trained by us), we employ a deep kernel that reduces the latent space to 12 dimensions before fitting the Gaussian Process. To mitigate vanishing gradients, we use log expected improvement \cite{ament2024unexpected} as our acquisition function, which is sequentially maximized.

\paragraph{Optimization tasks}
The Expressions dataset consists of arithmetic expressions that are functions of a single variable (e.g., $\sin(x)$, $1 + x*x$). Our objective is to find an expression that approximates $\texttt{1/3 + x + sin(x * x)}$, as described by \cite{kusner2017grammar}. The optimization target is defined as $\obj(\bx) = -\log(1 + \text{MSE}(\bx))$, where $\text{MSE}(\bx)$ is the mean-squared error between the expression $\bx$ and $\texttt{1/3 + x + sin(x * x)}$, evaluated over the range -10 to 10 using a grid of 1000 equally spaced points.

For the SMILES dataset, our goal is to maximize the octanol-water partition coefficient, which is calculated using the prediction model developed by \cite{wildman1999prediction}. In the case of the SELFIES dataset, following \cite{maus2022local}, we focus on three objectives: Perindopril MPO, Ranolazine MPO, and Zaleplon MPO, all of which are part of the Guacamol benchmarks \cite{brown2019guacamol}. While the SELFIES syntax is 100\% robust, we consider only solutions that pass a set of quality filters for evaluation (see \cref{ex:filters} for more details).

\paragraph{$\score$-constrained LSO}
For ease of implementation and numerical stability, when applying $\score$ regularization, we adopt a simple optimization procedure in which the acquisition function is optimized using 10 steps of normalized gradient ascent. This is because we empirically find that the norm of the derivative of the constraint (i.e., $\scoref(\bz)$) is typically much larger than the norm of the derivative of the acquisition function. As a result, using the gradient ascent update rule $\bz^{(i+1)} = \partial \acq(\sur(\bz^{(i)})) + \lambda \partial \scoref(\bz^{(i)})$ leads to a numerically unstable optimization process. To address this issue, we propose the following update rule: \begin{align} \label{eq:update_rule}\bz^{(i+1)} = \frac{\partial \acq(\sur(\bz^{(i)}))}{\Vert \partial \acq(\sur(\bz^{(i)}))\Vert_2} + \lambda \frac{\partial \scorefrho(\bz^{(i)})}{\Vert \partial \scorefrho(\bz^{(i)})\Vert_2}. \end{align}

We set $\lambda=0.05$ for Expressions, $\lambda=0.1$ for our SELFIES models, and $\lambda=0.5$ for SMILES and the pre-trained SELFIES-VAE. \change{For $\rano$ with pre-trained SELFIES-VAE, we use $\lambda=0.1$ to prevent over-regularization, as regularized methods already yield a high percentage of valid solutions (\cref{tab:validity})}.  We select these values because, without regularization, the SMILES and pre-trained SELFIES models tend to produce a lower percentage of valid solutions (\cref{tab:validity}). Based on our findings, we suggest $\lambda=0.5$ as a reasonable default value, while leaving the exploration of an optimal choice for future work.

\paragraph{Benchmark methods}
We compare our $\score$-constrained method (\textbf{\score}) with five alternative approaches for optimizing the acquisition function. First, we evaluate a non-regularized version of \cref{eq:update_rule}
, where $\lambda=0$ (\textbf{\noreg}), and a two regularization methods that uses the prior density (i.e., $\ell_2$ regularization) and likelihood  (\cref{eq:lik}) instead of $\score$ (\textbf{\prior}, and \textbf{\likelihood} respectively), using similar $\lambda$ values described above.

Additionally, we consider optimizing the acquisition function using the Limited-memory BFGS method within a symmetric hypercube centered at 0 (\textbf{\bound}), which is the default approach in the \texttt{BoTorch} package \cite{balandat2020botorch}. Since recent state-of-the-art LSO pipelines \cite{maus2022local, lee2024advancing} have utilized trust regions centered around the best observed value \cite{eriksson2019scalable}, we also compare with this approach (\textbf{\turbo}). Lastly, we implement the Uncertainty Censoring (\textbf{\uc}) method proposed by \cite{NEURIPS2021_06fe1c23}, which suggests early stopping of the optimization when the estimated uncertainty exceeds a certain threshold. For this threshold, we use the 99th percentile of the observed uncertainty values in the training data, as recommended by \cite{NEURIPS2021_06fe1c23}.

\paragraph{Hyperparameters} 
We calibrate the step size, which affects \textbf{\score}, \textbf{\uc}, \textbf{\prior}, \textbf{\likelihood}, and \textbf{\noreg}, to ensure our gradient ascent procedure (with $\lambda=0$) improves the acquisition function values across different initializations. The same step size is applied to all models within the same dataset: Expressions = 0.8, SMILES = 0.003, SELFIES = 0.03, and SELFIES pre-trained = 0.3. The $\bound$ method has a single hyperparameter, the facet length, which is set to 5. For $\turbo$, there are three primary hyperparameters: the initial length, which we set to 0.8, along with the success and failure tolerances, determining when to expand or shrink the trust region, set at 10 and 2, respectively. An ablation study is provided in \cref{sec:ablation}.

\subsection{Results}

\paragraph{Optimization results}  The experimental results, presenting the average across 10 independent LSO runs for the top 20 and best solutions found, are summarized in \cref{tab:top20,tab:top1} respectively. In both cases, $\score$ achieves the average best performance most frequently (22 and 17 out of 30 times, respectively). Furthermore, $\score$ outperforms other methods in both the average ranking across LSO tasks and the frequency with which it falls within one standard deviation of the best-performing method (\cref{tab:perf_short}). These findings demonstrate that using $\score$ as a regularization technique generally enhances optimization performance.  

\change{\Cref{fig:cumobj1,fig:cumobj20} show the cumulative  average top-20 and the best objective values during BO with the pre-trained SELFIES-VAE \cite{maus2022local}, respectiveley. Results align with \cite{maus2022local} under our realistic evaluation budget. For the average of the top-20 solutions on the 
$\rano$ and $\zale$ tasks, $\score$ outperforms $\turbo$ on valid solutions but underperforms overall, highlighting that $\score$ effectively constrains the optimzation.}

\begin{figure}
    \centering
    \includegraphics[width=0.6\linewidth]{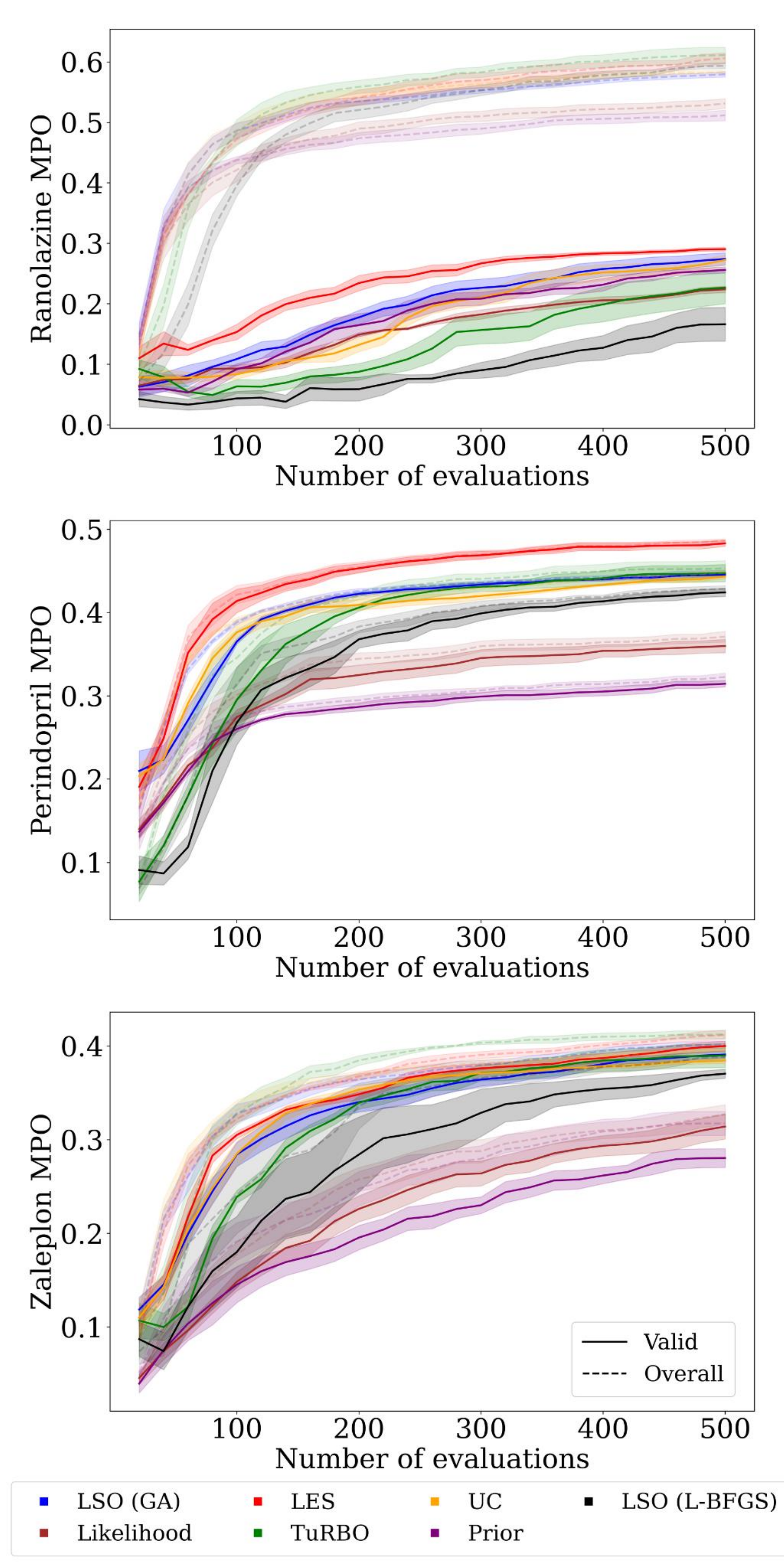}
    \caption{\change{Cumulative objective for the top-20 solutions found during Bayesian optimization with the pre-trained SELFIES-VAE (\cite{maus2022local}). Each method is shown in a distinct color. Solid lines represent solutions passing quality filters, while dashed lines include all evaluations. $\score$ outperforms all baselines on $\rano$ and $\pdop$, achieving competitive results on $\zale$.}}
    \label{fig:cumobj20}
\end{figure}

\begin{table}[h]
    \centering
    \renewcommand{\arraystretch}{1.2}
    \caption{\change{Summary metrics for 30 Bayesian Optimization experiments from \cref{tab:top1,tab:top20}. We report the average rank (lower is better) and the count of times each method is within one standard deviation of the best performing method (higher is better). Results cover both the best solution and the top-20 average. $\score$ outperforms alternatives on both metrics.}\\}
        \small

    \begin{tabular}{lcccc}
        \toprule
        & \multicolumn{2}{c}{\textbf{Top 1}} & \multicolumn{2}{c}{\textbf{Top 20}} \\
        \cmidrule(lr){2-3} \cmidrule(lr){4-5}
        \textbf{Method} & \textbf{Avg Rank} & \textbf{\# within std} & \textbf{Avg Rank} & \textbf{\# within std} \\
        \midrule
        $\score$       & \textbf{2.15}  & \textbf{19}  & \textbf{1.9}  & \textbf{22}  \\
        $\likelihood$  & \underline{2.76}  & \underline{16}  & \underline{2.53}  & \underline{13}  \\
        $\noreg$       & 3.15           & 8            & 2.87          & 7            \\
        $\prior$       & 3.86           & 10           & 3.8            & 5            \\
        $\uc$          & 4.17            & 5            & 4.3            & 3            \\
        $\turbo$       & 5.7           & 2            & 6.1           & 0            \\

        $\bound$       & 6.17           & 1            & 6.4          & 0            \\
        \bottomrule
    \end{tabular}

    \label{tab:perf_short}
\end{table}

\cref{tab:validity} shows the percentage of valid solutions found by each method across datasets and VAEs. $\score$ improves upon the non-regularized version of gradient ascent by 7\% on average and upon $\turbo$ and $\bound$ by 24\% and 36\% on average respectively. While $\uc$ achieves a 2\% higher percentage of valid solutions on average, we show in \cref{tab:exp_lambda} that for the Expressions datasets, where $\uc$ excels, the number of valid produced by $\score$ solutions can be increased by setting a higher $\lambda$ value. However, this did not improve optimization performance.

\section{Discussion}\label{sec:dicussion}
We proposed $\score$ to mitigate over-exploration in latent space optimization (LSO). $\score$ is differentiable and fully parallelizable. Extensive evaluations demonstrate that incorporating $\score$ as a penalty in LSO consistently enhances solution quality and objective outcomes. Moreover, $\score$ outperforms alternative regularization techniques, proving to be the most robust across diverse datasets and varying definitions of validity. In addition, $\score$ has only a single hyperparameter (the regularization strength), and we observe empirically that deploying $\score$ can provide significant performance gains. We therefore believe $\score$ offers a powerful approach for discovering more realistic solutions, when the criteria for realism are difficult to define or validate.

While $\score$ is fully parallelizable, it requires the calculation of the derivative of the decoder as well as the determinant of the change-of-variables term, which can be computationally expensive. This step can become a bottleneck when the size of the output and the latent dimension are both large. It is left for future work to develop a fast approximation for this operation in order to enable the use of $\score$ in applications involving large generative models.

\section*{Acknowledgements}
Ronen and Yu gratefully acknowledge the support of the NSF for FODSI through grants DMS-2023505 and DMS/NIGMS-R01GM152718 of the NSF and the Simons Foundation for the Collaboration on the Theoretical Foundations of Deep Learning through awards DMS-2031883 and \#814639, DMS-2210827, CCF-2315725  and of the ONR through MURI award N000142112431 and N00014-24-S-B001. 
Humayun and Baraniuk gratefully acknowledge the support from NSF grants CCF-1911094, IIS-1838177, and IIS-1730574; ONR grants N00014-18-1-2571, N00014-20-1-2534, N00014-23-1-2714, and MURI N00014-20-1-2787; AFOSR grant FA9550-22-1-0060; DOI grant 140D0423C0076; and a Vannevar Bush Faculty Fellowship, ONR grant
N00014-18-1-2047.
\newpage

\newpage
\bibliography{main_arxiv}

\newpage
\appendix

\section{Proofs}\label{sec:proofs}
\begin{lemma}\label{lem:jac}
    Let $f_\theta$ be a DGN as defined in \cref{eq:extended_dgn} and assume that $f_\theta$ can be expressed as a CPA (\cref{eq:nn_cpas}) and is inevitable, then
    \begin{align}
        J{f^{-1}_\theta}(\bx) &=  \left(\begin{bmatrix}
\bs{B}_{1} & \cdots & \bs{0} \\
\vdots & \ddots & \vdots \\
\bs{0} & \cdots & \bs{B}_{\seqlen}
\end{bmatrix}\bs{A}_{\omega}^{\dagger}\right)^T,
    \end{align}
    where $\bs{A}_{\omega}^{\dagger}$ is the Moore Penrose inverse of the slope matrix, at the knot whose image constrains $\bx$, and
\begin{align}
   \bs{B}_{i} &=  \left(\text{diag}\left(\frac{1}{(\bs{p}^{(i)}_{\bs{z}})_{1}}, \dots, \frac{1}{(\bs{p}^{(i)}_{\bs{z}})_{\vocsize}}\right), -\bs{1}\frac{1}{c^{(i)}_{\bs{z}}}\right)^T.
\end{align}
\end{lemma}

\begin{proof}
    First we write 
    \begin{align}
        f_\theta (\bz) &= \text{Softmax}_{+}(\ell_\theta (\bz)),
    \end{align}

    Where $\text{Softmax}_{+}$ is the extension of the column wise Softmax function to include the normalizing constants. Specifically, for $\seqlen$ by $\vocsize$ $\ell_\theta(\bz)$ matrix, we have
    \begin{align}
        \text{Softmax}_{+}(\ell_\theta(\bz)) &= \left(\bp^{(1)}_{\bz}, (c^{(1)}_{\bz})^{-1}, \dots, \bp^{(L)}_{\bz}, (c^{(L)}_{\bz})^{-1}\right) = \bx_{\bz},
    \end{align} 
    with $\bp^{(i)}_{\bz} = (\frac{e^{\ell_\theta(\bz)_{1i}}}{c^{(i)}_{\bz}})$, and $c^{(i)}_{\bz} = \sum_{j=1}^\vocsize \exp({\ell_\theta(\bz)_{ji}})$.\\
    Next,
    \begin{align}
        f_\theta^{(-1)}(\bx) &= \ell^{-1}_\theta (\text{Softmax}_{+}^{-1}(\bx))
    \end{align}
   A direct calculation yields, 
\begin{align}
    \text{Softmax}_{+}^{-1}(\bx) &= \left(\log(\bs{p}^{(1)}_{\bs{z}}) + \log(c^{(1)}_{\bs{z}}), \dots, \log(\bs{p}^{(L)}_{\bs{z}}) + \log(c^{(L)}_{\bs{z}})\right).
\end{align}
    As we assume $\ell_{\theta}$ is bijective and can be written as 
    \begin{align}
        \ell_{\theta} (\bz) &= \sum_{\omega \in \Omega} \left(\bA_{\omega}\bz + \bb_{\omega}\right) 1_{\bz \in \omega}, 
    \end{align}
   we have that 
   \begin{align}
       \ell^{-1}_\theta (\text{Softmax}_{+}^{-1}(\bx)) &=  (\text{Softmax}_{+}^{-1}(\bx) - \bb_{\omega}) \bA_{\omega}^\dagger.
   \end{align}
    Lastly, as 
    \begin{align}
        \frac{\partial \text{Softmax}_{+}^{-1}(\bx)}{\partial \bx} &= \begin{bmatrix}
\bs{B}_{1} & \cdots & \bs{0} \\
\vdots & \ddots & \vdots \\
\bs{0} & \cdots & \bs{B}_{\seqlen}
\end{bmatrix}, 
    \end{align}
    for 
    \begin{align}
   \bs{B}_{i} &=  \left(\text{diag}\left(\frac{1}{(\bs{p}^{(i)}_{\bs{z}})_{1}}, \dots, \frac{1}{(\bs{p}^{(i)}_{\bs{z}})_{\vocsize}}\right), -\bs{1}\frac{1}{c^{(i)}_{\bs{z}}}\right)^T.
\end{align}

    we obtain the final result.
\end{proof}
    
\begin{proof}[Proof of \cref{thm:img_den}]\label{prf:img_den}
First, we note that by our invertibility assumption we have that $\P(\bx \in W) = \P(\bz \in f_\theta^{(-1)} (W))$. We then proceed with a direct calculation
\begin{align}
    \P(\bx \in W) &= \P(\bz \in f_\theta^{(-1)}(W))\\
                &= \sum_{\omega \in \Omega}  \P(\bz \in (f_\theta^{(-1)}(W)\cap \omega))\\
                &=  \sum_{\omega \in \Omega} \intop_{f_\theta^{(-1)}(W)\cap \omega}f_{\bz}(\bz) d\bz\\
                &= \sum_{\omega \in \Omega}\intop_{W\cap f_\theta(\omega)}f_{\bz}(f_\theta^{(-1)}(\bx)) \sqrt{\det{\left(Jf_\theta^{(-1)}(\bx) Jf_\theta^{(-1)}(\bx)^T\right)}} d\bx\label{eq:cov}\\
                &=\intop_{W} \sum_{\omega \in \Omega} f_{\bz}(f_\theta^{(-1)}(\bx)) \sqrt{\det{\left(Jf_\theta^{(-1)}(\bx) Jf_\theta^{(-1)}(\bx)^T\right)}}1_{\{\bx\in f_{\theta}(\omega)\}} d\bx.
\end{align}
Using \cref{lem:jac}, we get that the volume element is
\begin{align}
    Jf_\theta^{(-1)}(\bx) Jf_\theta^{(-1)}(\bx)^T &=
     \left(\begin{bmatrix}
\bs{B}_{1} & \cdots & \bs{0} \\
\vdots & \ddots & \vdots \\
\bs{0} & \cdots & \bs{B}_{\seqlen}
\end{bmatrix}\bs{A}_{\omega}^{\dagger}\right)^T \left(\begin{bmatrix}
\bs{B}_{1} & \cdots & \bs{0} \\
\vdots & \ddots & \vdots \\
\bs{0} & \cdots & \bs{B}_{\seqlen}
\end{bmatrix}\bs{A}_{\omega}^{\dagger}\right)\\&
    \left((\bA_{\omega}^\dagger)^T\begin{bmatrix}
\bs{B}_{1}^T & \cdots & \bs{0} \\
\vdots & \ddots & \vdots \\
\bs{0} & \cdots & \bs{B}_{\seqlen}^T,
\end{bmatrix} \right) \left(\begin{bmatrix}
\bs{B}_{1} & \cdots & \bs{0} \\
\vdots & \ddots & \vdots \\
\bs{0} & \cdots & \bs{B}_{\seqlen}
\end{bmatrix}\bs{A}_{\omega}^{\dagger}\right)\\
&=  \sum_{i=1}^\seqlen  (\bs{A}^{\dagger}_{i})^T (\bs{B}_{i})^T\bs{B}_{i}  \bs{A}^{\dagger}_{i},
\end{align}

where $\bs{A}^{\dagger}_{i} = \left(\bs{A}^{(1)}_{\omega}, \dots, \bs{A}^{(L)}_{\omega}\right)^{\dagger}_{(i\cdot \vocsize) : (i+1 \cdot \vocsize).}$.
\end{proof}

\begin{lemma}\label{lem:mes_0}
    Let $f:\mathcal{Z}\rightarrow\mathcal{X}$ be a function and define $f^\dagger:\mathcal{X}\rightarrow\mathcal{Z}$ as $f^\dagger(x) \in \{z: f(z) = x\}$.  Let $\mu$ be the Lebesgue measure and assume that $\mu(\{z; \exists z' \;\text{s.t.}\; f(z) = f(z')\}) = 0$, then for every $B\subseteq \mathcal{Z}$ we have
    \begin{align}
        \mu(\{z; f(z)\in B\}) = \mu(f^\dagger (B))
    \end{align}
\end{lemma}

\begin{proof}
    We proceed with direct calculation
    \begin{align}
        \mu(\{z; f(z)\in B\}) &\leq \mu(f^\dagger(B)) + \mu(\{z; \exists z' \;\text{s.t.}\; f(z) = f(z')\})
    \end{align}
    Now assume that $\mu(\{z; \exists z' \; \text{s.t.} \; f(z) = f(z')\})=0$, we have that $ \mu(\{z; f(z)\in B\})
    \leq\mu(f^\dagger(B))$. The other direction follows immediately from the definition of $f^\dagger$.
\end{proof}

\cref{lem:mes_0} implies that \cref{eq:cov} can still hold under the assumption that 
$\mu(\{z; \exists z' \;\text{s.t.}\; f_\theta(z) = f_\theta(z')\}) = 0$.

\section{Ablation studies}\label{sec:ablation}
\begin{table}[H]
\centering
\renewcommand{\arraystretch}{1.2} 
\caption{Ablation study for the initial length, success tolerance and failure tolerance for the $\turbo$ method. Average across 10 independent runs, of the best value across datasets, architectures. Column names indicate the length/success/fail values. Results from the main paper are in bold.\\}
\scriptsize
\begin{tabular}{lcccccccccc}
\toprule
 & \textbf{Architecture} & \textbf{$\beta$} & \textbf{0.8/10/10} & \textbf{0.8/10/2} & \textbf{0.8/2/10} & \textbf{0.8/2/2} & \textbf{1.6/10/10} & \textbf{1.6/10/2} & \textbf{1.6/2/10} & \textbf{1.6/2/2} \\
\midrule
\multirow{9}{*}{\rotatebox{90}{\textbf{Expressions}}}
& \multirow{3}{*}{GRU} 
    & 0.05 & -0.65 (0.11) & \textbf{-0.73 (0.12)} & -0.71 (0.09) & -0.61 (0.1) & -0.65 (0.12) & -0.59 (0.06) & -0.62 (0.11) & -0.72 (0.13) \\
& 
    & 0.1 & -0.56 (0.05) & \textbf{-0.61 (0.07)} & -0.59 (0.07) & -0.71 (0.1) & -0.54 (0.04) & -0.58 (0.05) & -0.62 (0.04) & -0.63 (0.04) \\
& 
    & 1 & -0.56 (0.05) & \textbf{-0.54 (0.05)} & -0.54 (0.06) & -0.6 (0.08) & -0.56 (0.06) & -0.61 (0.05) & -0.6 (0.04) & -0.59 (0.05) \\
\cmidrule{2-11}
& \multirow{3}{*}{LSTM} 
    & 0.05 & -0.46 (0.03) & \textbf{-0.43 (0.02)} & -0.43 (0.02) & -0.38 (0.02) & -0.43 (0.02) & -0.43 (0.02) & -0.47 (0.04) & -0.4 (0.01) \\
& 
    & 0.1 & -0.38 (0.04) & \textbf{-0.39 (0.0)} & -0.41 (0.01) & -0.42 (0.02) & -0.42 (0.02) & -0.42 (0.02) & -0.42 (0.02) & -0.42 (0.02) \\
& 
    & 1 & -0.96 (0.09) & \textbf{-0.86 (0.0)} & -0.86 (0.0) & -1.01 (0.11) & -0.88 (0.01) & -0.98 (0.06) & -0.92 (0.04) & -0.88 (0.01) \\
\cmidrule{2-11}
& \multirow{3}{*}{Transformer} 
    & 0.05 & -0.39 (0.04) & \textbf{-0.44 (0.02)} & -0.39 (0.04) & -0.4 (0.05) & -0.44 (0.02) & -0.39 (0.04) & -0.38 (0.04) & -0.42 (0.02) \\
& 
    & 0.1 & -0.38 (0.04) & \textbf{-0.41 (0.02)} & -0.42 (0.02) & -0.42 (0.02) & -0.39 (0.04) & -0.41 (0.01) & -0.41 (0.02) & -0.37 (0.04) \\
& 
    & 1 & -0.67 (0.1) & \textbf{-0.58 (0.1)} & -0.52 (0.08) & -0.62 (0.06) & -0.66 (0.11) & -0.62 (0.07) & -0.56 (0.05) & -0.54 (0.05) \\
\midrule
\multirow{9}{*}{\rotatebox{90}{\textbf{SELFIES}}}
& \multirow{3}{*}{Transformer (pdop)} 
    & 0.05 & 0.13 (0.02) & \textbf{0.15 (0.03)} & 0.17 (0.03) & 0.18 (0.03) & 0.23 (0.03) & 0.17 (0.04) & 0.22 (0.04) & 0.2 (0.04) \\
& 
    & 0.1 & 0.08 (0.02) & \textbf{0.09 (0.03)} & 0.1 (0.03) & 0.09 (0.03) & 0.1 (0.03) & 0.12 (0.04) & 0.1 (0.03) & 0.14 (0.04) \\
& 
    & 1 & 0.36 (0.02) & \textbf{0.31 (0.03)} & 0.36 (0.02) & 0.34 (0.02) & 0.38 (0.01) & 0.38 (0.0) & 0.4 (0.01) & 0.36 (0.02) \\
\cmidrule{2-11}
& \multirow{3}{*}{Transformer (rano)} 
    & 0.05 & 0.17 (0.02) & \textbf{0.2 (0.02)} & 0.2 (0.02) & 0.18 (0.02) & 0.11 (0.02) & 0.08 (0.01) & 0.12 (0.02) & 0.13 (0.02) \\
& 
    & 0.1 & 0.09 (0.02) & \textbf{0.1 (0.01)} & 0.11 (0.02) & 0.1 (0.02) & 0.06 (0.01) & 0.05 (0.01) & 0.04 (0.01) & 0.06 (0.01) \\
& 
    & 1 & 0.07 (0.02) & \textbf{0.05 (0.02)} & 0.11 (0.03) & 0.07 (0.02) & 0.16 (0.02) & 0.16 (0.03) & 0.11 (0.02) & 0.12 (0.02) \\
\cmidrule{2-11}
& \multirow{3}{*}{Transformer (zale)} 
    & 0.05 & 0.11 (0.02) & \textbf{0.15 (0.03)} & 0.12 (0.02) & 0.16 (0.03) & 0.22 (0.03) & 0.23 (0.03) & 0.21 (0.04) & 0.19 (0.03) \\
& 
    & 0.1 & 0.14 (0.01) & \textbf{0.16 (0.01)} & 0.15 (0.02) & 0.14 (0.01) & 0.1 (0.02) & 0.12 (0.02) & 0.13 (0.02) & 0.14 (0.03) \\
& 
    & 1 & 0.36 (0.02) & \textbf{0.31 (0.03)} & 0.38 (0.02) & 0.36 (0.02) & 0.38 (0.01) & 0.34 (0.02) & 0.31 (0.03) & 0.38 (0.02) \\
\midrule
\multirow{3}{*}{\rotatebox{90}{\textbf{SELFIES (\cite{maus2022local})}}}
& \multirow{1}{*}{Transformer (pdop)} 
    & 1 & 0.48 (0.02) & \textbf{0.48 (0.02)} & 0.53 (0.03) & 0.49 (0.02) & 0.51 (0.02) & 0.51 (0.02) & 0.49 (0.02) & 0.50 (0.01) \\
\cmidrule{2-11}
& \multirow{1}{*}{Transformer (rano)} 
    & 1 & 0.38 (0.01) & \textbf{0.35 (0.01)} & 0.35 (0.01) & 0.32 (0.01) & 0.36 (0.01) & 0.36 (0.01) & 0.37 (0.01) & 0.34 (0.01) \\
\cmidrule{2-11}
& \multirow{1}{*}{Transformer (zale)} 
    & 1 & 0.44 (0.02) & \textbf{0.44 (0.01)} & 0.44 (0.03) & 0.49 (0.02) & 0.47 (0.01) & 0.46 (0.01) & 0.49 (0.01) & 0.47 (0.01) \\
\midrule
\multirow{9}{*}{\rotatebox{90}{\textbf{SMILES}}}
& \multirow{3}{*}{GRU} 
    & 0.05 & 2.47 (0.22) & \textbf{2.47 (0.22)} & 2.42 (0.21) & 2.24 (0.17) & 2.35 (0.23) & 2.42 (0.28) & 1.97 (0.18) & 2.25 (0.31) \\
& 
    & 0.1 & 1.76 (0.3) & \textbf{2.57 (0.31)} & 2.45 (0.26) & 2.07 (0.34) & 2.24 (0.28) & 2.26 (0.33) & 2.24 (0.31) & 1.92 (0.24) \\
& 
    & 1 & 2.43 (0.32) & \textbf{2.48 (0.29)} & 2.76 (0.24) & 2.17 (0.49) & 2.43 (0.32) & 2.35 (0.37) & 1.4 (0.52) & 2.04 (0.33) \\
\cmidrule{2-11}
& \multirow{3}{*}{LSTM} 
    & 0.05 & 2.65 (0.32) & \textbf{2.73 (0.33)} & 2.89 (0.26) & 3.02 (0.34) & 2.64 (0.3) & 2.91 (0.29) & 2.77 (0.23) & 2.68 (0.37) \\
& 
    & 0.1 & 1.97 (0.37) & \textbf{1.78 (0.43)} & 1.98 (0.4) & 2.16 (0.29) & 2.36 (0.41) & 1.87 (0.39) & 2.34 (0.33) & 1.73 (0.41) \\
& 
    & 1 & 3.06 (0.2) & \textbf{2.71 (0.34)} & 2.65 (0.16) & 2.83 (0.25) & 2.75 (0.28) & 3.49 (0.26) & 2.68 (0.31) & 3.16 (0.35) \\
\cmidrule{2-11}
& \multirow{3}{*}{Transformer} 
    & 0.05 & 2.47 (0.23) & \textbf{2.88 (0.24)} & 2.42 (0.17) & 2.67 (0.27) & 2.91 (0.3) & 2.25 (0.2) & 2.43 (0.24) & 2.48 (0.29) \\
& 
    & 0.1 & 3.03 (0.3) & \textbf{2.15 (0.18)} & 2.51 (0.25) & 2.41 (0.24) & 2.28 (0.12) & 2.28 (0.22) & 2.53 (0.26) & 2.46 (0.16) \\
& 
    & 1 & 2.37 (0.21) & \textbf{2.25 (0.18)} & 2.16 (0.17) & 2.21 (0.25) & 2.46 (0.2) & 2.11 (0.12) & 2.31 (0.16) & 2.71 (0.32) \\
\bottomrule
\end{tabular}
\end{table}

\begin{table}[H]
\centering
\renewcommand{\arraystretch}{1.2} 
\caption{Ablation study for the initial length, success tolerance and failure tolerance for the $\turbo$ method. Average across 10 independent runs, of the top 20 best values across datasets, architectures. Column names indicate the length/success/fail values. Results from the main paper are in bold.\\}
\scriptsize
\begin{tabular}{lcccccccccc}
\toprule
 & \textbf{Architecture} & \textbf{$\beta$} & \textbf{0.8/10/10} & \textbf{0.8/10/2} & \textbf{0.8/2/10} & \textbf{0.8/2/2} & \textbf{1.6/10/10} & \textbf{1.6/10/2} & \textbf{1.6/2/10} & \textbf{1.6/2/2} \\
\midrule
\multirow{9}{*}{\rotatebox{90}{\textbf{Expressions}}}
& \multirow{3}{*}{GRU} 
    & 0.05 & -2.03 (0.12) & \textbf{-2.1 (0.11)} & -2.0 (0.09) & -2.0 (0.11) & -1.93 (0.09) & -1.92 (0.08) & -1.91 (0.06) & -1.93 (0.07) \\
& 
    & 0.1 & -2.18 (0.15) & \textbf{-2.14 (0.18)} & -2.11 (0.2) & -2.14 (0.2) & -1.98 (0.13) & -2.05 (0.12) & -2.0 (0.17) & -2.09 (0.22) \\
& 
    & 1 & -1.39 (0.12) & \textbf{-1.34 (0.05)} & -1.56 (0.11) & -1.61 (0.11) & -1.36 (0.08) & -1.42 (0.11) & -1.45 (0.12) & -1.38 (0.12) \\
\cmidrule{2-11}
& \multirow{3}{*}{LSTM} 
    & 0.05 & -1.45 (0.09) & \textbf{-1.59 (0.13)} & -1.63 (0.12) & -1.49 (0.1) & -1.46 (0.08) & -1.52 (0.08) & -1.51 (0.12) & -1.31 (0.07) \\
& 
    & 0.1 & -1.31 (0.12) & \textbf{-1.32 (0.13)} & -1.37 (0.11) & -1.31 (0.11) & -1.31 (0.1) & -1.36 (0.13) & -1.36 (0.11) & -1.35 (0.09) \\
& 
    & 1 & -2.16 (0.06) & \textbf{-2.22 (0.08)} & -2.23 (0.08) & -2.23 (0.08) & -2.13 (0.07) & -2.16 (0.06) & -2.15 (0.05) & -2.08 (0.06) \\
\cmidrule{2-11}
& \multirow{3}{*}{Transformer} 
    & 0.05 & -2.07 (0.12) & \textbf{-2.16 (0.09)} & -2.03 (0.17) & -2.15 (0.14) & -1.96 (0.13) & -1.74 (0.13) & -2.09 (0.1) & -1.79 (0.11) \\
& 
    & 0.1 & -1.45 (0.13) & \textbf{-1.39 (0.12)} & -1.56 (0.13) & -1.38 (0.13) & -1.32 (0.11) & -1.43 (0.17) & -1.4 (0.1) & -1.33 (0.12) \\
& 
    & 1 & -1.96 (0.11) & \textbf{-1.77 (0.12)} & -1.81 (0.11) & -1.85 (0.09) & -1.85 (0.11) & -1.76 (0.08) & -1.78 (0.08) & -1.75 (0.09) \\
\midrule
\multirow{9}{*}{\rotatebox{90}{\textbf{SELFIES}}}
& \multirow{3}{*}{Transformer (pdop)} 
    & 0.05 & 0.02 (0.01) & \textbf{0.04 (0.01)} & 0.04 (0.01) & 0.04 (0.01) & 0.09 (0.01) & 0.09 (0.03) & 0.08 (0.01) & 0.09 (0.02) \\
& 
    & 0.1 & 0.01 (0.0) & \textbf{0.01 (0.0)} & 0.02 (0.01) & 0.02 (0.01) & 0.03 (0.01) & 0.08 (0.03) & 0.04 (0.02) & 0.05 (0.02) \\
& 
    & 1 & 0.24 (0.01) & \textbf{0.23 (0.02)} & 0.21 (0.02) & 0.24 (0.01) & 0.27 (0.01) & 0.26 (0.01) & 0.28 (0.01) & 0.26 (0.01) \\
\cmidrule{2-11}
& \multirow{3}{*}{Transformer (rano)} 
    & 0.05 & 0.06 (0.01) & \textbf{0.06 (0.0)} & 0.06 (0.01) & 0.06 (0.01) & 0.04 (0.01) & 0.02 (0.0) & 0.04 (0.0) & 0.03 (0.01) \\
& 
    & 0.1 & 0.03 (0.01) & \textbf{0.03 (0.01)} & 0.03 (0.0) & 0.02 (0.0) & -- & -- & -- & -- \\
& 
    & 1 & -- & \textbf{--} & 0.08 (0.0) & -- & 0.05 (0.01) & 0.07 (0.01) & 0.07 (0.0) & -- \\
\cmidrule{2-11}
& \multirow{3}{*}{Transformer (zale)} 
    & 0.05 & 0.02 (0.01) & \textbf{0.04 (0.01)} & 0.02 (0.01) & 0.04 (0.01) & 0.07 (0.01) & 0.07 (0.02) & 0.06 (0.01) & 0.06 (0.01) \\
& 
    & 0.1 & 0.04 (0.01) & \textbf{0.04 (0.01)} & 0.04 (0.01) & 0.03 (0.01) & 0.01 (0.0) & 0.02 (0.0) & 0.09 (0.0) & 0.04 (0.0) \\
& 
    & 1 & 0.16 (0.02) & \textbf{0.17 (0.02)} & 0.16 (0.02) & 0.17 (0.02) & 0.18 (0.02) & 0.16 (0.02) & 0.18 (0.02) & 0.19 (0.02) \\
\midrule
\multirow{3}{*}{\rotatebox{90}{\textbf{SELFIES (\cite{maus2022local})}}}
& \multirow{1}{*}{Transformer (pdop)} 
    & 1 & 0.44 (0.01) & \textbf{0.44 (0.01)} & 0.44 (0.01) & 0.44 (0.01) & 0.45 (0.0) & 0.45 (0.01) & 0.44 (0.01) & 0.44 (0.01) \\
\cmidrule{2-11}
& \multirow{1}{*}{Transformer (rano)} 
    & 1 & 0.20 (0.02) & \textbf{0.22 (0.02)} & 0.21 (0.01) & 0.20 (0.01) & 0.23 (0.01) & 0.22 (0.01) & 0.22 (0.01) & 0.23 (0.01) \\
\cmidrule{2-11}
& \multirow{1}{*}{Transformer (zale)} 
    & 1 & 0.37 (0.01) & \textbf{0.37 (0.01)} & 0.36 (0.01) & 0.37 (0.01) & 0.39 (0.01) & 0.38 (0.01) & 0.39 (0.01) & 0.39 (0.01) \\
\midrule
\multirow{9}{*}{\rotatebox{90}{\textbf{SMILES}}}
& \multirow{3}{*}{GRU} 
    & 0.05 & 0.89 (0.19) & \textbf{0.82 (0.15)} & 0.64 (0.2) & 0.89 (0.14) & 1.0 (0.14) & 0.98 (0.19) & 0.7 (0.24) & 1.07 (0.19) \\
& 
    & 0.1 & 0.4 (0.22) & \textbf{0.41 (0.19)} & 0.63 (0.26) & 0.55 (0.14) & 0.83 (0.13) & 0.96 (0.33) & 0.69 (0.18) & 0.75 (0.18) \\
& 
    & 1 & 1.44 (0.0) & \textbf{-0.16 (0.0)} & -- & -- & 0.14 (0.65) & -- & -- & 1.14 (0.0) \\
\cmidrule{2-11}
& \multirow{3}{*}{LSTM} 
    & 0.05 & 1.18 (0.25) & \textbf{1.54 (0.41)} & 1.1 (0.15) & 1.1 (0.36) & 0.86 (0.27) & 1.05 (0.21) & 1.07 (0.29) & 1.22 (0.27) \\
& 
    & 0.1 & -- & \textbf{--} & -- & -- & 0.21 (0.0) & -- & -- & -- \\
& 
    & 1 & 0.61 (0.16) & \textbf{0.52 (0.26)} & 0.48 (0.13) & 0.7 (0.23) & 0.67 (0.23) & 1.13 (0.16) & 0.97 (0.28) & 0.88 (0.23) \\
\cmidrule{2-11}
& \multirow{3}{*}{Transformer} 
    & 0.05 & 0.88 (0.18) & \textbf{0.97 (0.19)} & 0.92 (0.14) & 0.67 (0.15) & 0.93 (0.17) & 1.16 (0.1) & 1.09 (0.11) & 0.97 (0.24) \\
& 
    & 0.1 & 0.42 (0.12) & \textbf{0.62 (0.14)} & 0.44 (0.16) & 0.73 (0.14) & 0.91 (0.1) & 0.6 (0.14) & 0.81 (0.14) & 0.7 (0.14) \\
& 
    & 1 & 0.72 (0.18) & \textbf{0.48 (0.18)} & 0.61 (0.15) & 0.47 (0.16) & 0.99 (0.12) & 0.74 (0.15) & 0.83 (0.16) & 0.68 (0.15) \\
\bottomrule
\end{tabular}
\end{table}

\begin{table}[H]
\centering
\renewcommand{\arraystretch}{1.2} 
\caption{Ablation study for the facet length parameter of $\bound$ method. Average across 10 independent runs of the average top 20 values across datasets, architectures, and bound methods are displayed for facet lengths of size 1, 5 and 10. Results from the main paper are bold.\\}
\small
\begin{tabular}{lccccc}
\toprule
 & \textbf{Architecture} & \textbf{$\beta$} &\textbf{1} & \textbf{5} & \textbf{10} \\
\midrule
\multirow{9}{*}{\rotatebox{90}{\textbf{Expressions}}}
& \multirow{3}{*}{GRU} 
    & 0.05 & -1.79 (0.08) & \textbf{-1.72 (0.07)} & -1.72 (0.07) \\
& 
    & 0.1 & -1.73 (0.11) & \textbf{-1.93 (0.09)} & -1.89 (0.11) \\
& 
    & 1 & -1.94 (0.07) & \textbf{-1.97 (0.07)} & -2.03 (0.09) \\
\cmidrule{2-6}
& \multirow{3}{*}{LSTM} 
    & 0.05 & -1.89 (0.09) & \textbf{-1.78 (0.09)} & -1.93 (0.06) \\
& 
    & 0.1 & -1.29 (0.06) & \textbf{-1.39 (0.08)} & -1.37 (0.06) \\
& 
    & 1 & -2.04 (0.05) & \textbf{-2.04 (0.04)} & -2.04 (0.05) \\
\cmidrule{2-6}
& \multirow{3}{*}{Transformer} 
    & 0.05 & -3.11 (0.14) & \textbf{-2.93 (0.13)} & -3.02 (0.09) \\
& 
    & 0.1 & -3.19 (0.28) & \textbf{-2.69 (0.25)} & -2.93 (0.22) \\
& 
    & 1 & -2.44 (0.11) & \textbf{-2.41 (0.09)} & -2.28 (0.11) \\
\midrule
\multirow{9}{*}{\rotatebox{90}{\textbf{SELFIES}}}
& \multirow{3}{*}{Transformer (pdop)} 
    & 0.05 & 0.22 (0.01) & \textbf{0.21 (0.01)} & 0.25 (0.01) \\
& 
    & 0.1 & 0.25 (0.01) & \textbf{0.19 (0.01)} & 0.22 (0.03) \\
& 
    & 1 & 0.26 (0.01) & \textbf{0.25 (0.01)} & 0.27 (0.01) \\
\cmidrule{2-6}
& \multirow{3}{*}{Transformer (rano)} 
    & 0.05 & 0.04 (0.0) & \textbf{0.03 (0.0)} & -- \\
& 
    & 0.1 & -- & \textbf{--} & -- \\
& 
    & 1 & 0.09 (0.02) & \textbf{0.07 (0.0)} & 0.08 (0.01) \\
\cmidrule{2-6}
& \multirow{3}{*}{Transformer (zale)} 
    & 0.05 & 0.12 (0.02) & \textbf{0.13 (0.01)} & 0.11 (0.01) \\
& 
    & 0.1 & 0.08 (0.01) & \textbf{0.06 (0.0)} & 0.07 (0.0) \\
& 
    & 1 & 0.18 (0.02) & \textbf{0.18 (0.02)} & 0.19 (0.02) \\
\midrule
\multirow{3}{*}{\rotatebox{90}{\textbf{SELFIES (\cite{maus2022local})}}}
& \multirow{1}{*}{Transformer (pdop)} 
    & 1 & 0.45 (0.00) & \textbf{0.42 (0.00)} & 0.43 (0.01) \\
\cmidrule{2-6}
& \multirow{1}{*}{Transformer (rano)} 
    & 1 & 0.22 (0.01) & \textbf{0.17 (0.01)} & 0.19 (0.01) \\
\cmidrule{2-6}
& \multirow{1}{*}{Transformer (zale)} 
    & 1 & 0.4 (0.01) & \textbf{0.37 (0.01)} & 0.38 (0.01) \\
\midrule
\multirow{9}{*}{\rotatebox{90}{\textbf{SMILES}}}
& \multirow{3}{*}{GRU} 
    & 0.05 & 0.65 (0.12) & \textbf{0.52 (0.12)} & 0.52 (0.13) \\
& 
    & 0.1 & 0.52 (0.14) & \textbf{0.12 (0.14)} & 0.36 (0.15) \\
& 
    & 1 & -- & \textbf{--} & -- \\
\cmidrule{2-6}
& \multirow{3}{*}{LSTM} 
    & 0.05 & 0.67 (0.15) & \textbf{0.66 (0.12)} & 0.56 (0.2) \\
& 
    & 0.1 & -- & \textbf{--} & -- \\
& 
    & 1 & 0.49 (0.16) & \textbf{0.7 (0.2)} & 0.55 (0.15) \\
\cmidrule{2-6}
& \multirow{3}{*}{Transformer} 
    & 0.05 & 0.52 (0.18) & \textbf{0.42 (0.15)} & 0.55 (0.12) \\
& 
    & 0.1 & 0.63 (0.18) & \textbf{0.61 (0.15)} & 0.56 (0.17) \\
& 
    & 1 & 0.16 (0.15) & \textbf{0.23 (0.12)} & 0.39 (0.13) \\
\bottomrule
\end{tabular}
\end{table}

\begin{table}[H]
\centering
\renewcommand{\arraystretch}{1.2} 
\caption{Ablation study for the facet length parameter of $\bound$ method. Average across 10 independent runs, of the best value across datasets, architectures, and bound methods are displayed for facet lengths of size 1, 5 and 10. Results from the main paper are bold.\\}
\small
\begin{tabular}{lccccc}
\toprule
 & \textbf{Architecture} & \textbf{$\beta$} &\textbf{1} & \textbf{5} & \textbf{10} \\
\midrule
\multirow{9}{*}{\rotatebox{90}{\textbf{Expressions}}}
& \multirow{3}{*}{GRU} 
    & 0.05 & -0.57 (0.07) & \textbf{-0.56 (0.04)} & -0.56 (0.07) \\
& 
    & 0.1 & -0.53 (0.04) & \textbf{-0.46 (0.02)} & -0.46 (0.02) \\
& 
    & 1 & -0.68 (0.05) & \textbf{-0.77 (0.02)} & -0.73 (0.04) \\
\cmidrule{2-6}
& \multirow{3}{*}{LSTM} 
    & 0.05 & -0.57 (0.11) & \textbf{-0.56 (0.04)} & -0.67 (0.09) \\
& 
    & 0.1 & -0.4 (0.05) & \textbf{-0.44 (0.04)} & -0.47 (0.04) \\
& 
    & 1 & -1.01 (0.1) & \textbf{-1.02 (0.07)} & -0.86 (0.0) \\
\cmidrule{2-6}
& \multirow{3}{*}{Transformer} 
    & 0.05 & -1.06 (0.13) & \textbf{-0.8 (0.1)} & -1.11 (0.18) \\
& 
    & 0.1 & -0.8 (0.14) & \textbf{-0.65 (0.1)} & -0.69 (0.09) \\
& 
    & 1 & -0.85 (0.1) & \textbf{-0.82 (0.1)} & -0.78 (0.12) \\
\midrule
\multirow{9}{*}{\rotatebox{90}{\textbf{SELFIES}}}
& \multirow{3}{*}{Transformer (pdop)} 
    & 0.05 & 0.36 (0.01) & \textbf{0.36 (0.02)} & 0.34 (0.03) \\
& 
    & 0.1 & 0.29 (0.03) & \textbf{0.34 (0.02)} & 0.27 (0.03) \\
& 
    & 1 & 0.39 (0.01) & \textbf{0.36 (0.02)} & 0.35 (0.03) \\
\cmidrule{2-6}
& \multirow{3}{*}{Transformer (rano)} 
    & 0.05 & 0.06 (0.01) & \textbf{0.07 (0.01)} & 0.08 (0.01) \\
& 
    & 0.1 & 0.08 (0.02) & \textbf{0.08 (0.02)} & 0.07 (0.01) \\
& 
    & 1 & 0.16 (0.02) & \textbf{0.16 (0.02)} & 0.16 (0.02) \\
\cmidrule{2-6}
& \multirow{3}{*}{Transformer (zale)} 
    & 0.05 & 0.27 (0.03) & \textbf{0.27 (0.03)} & 0.26 (0.03) \\
& 
    & 0.1 & 0.15 (0.04) & \textbf{0.19 (0.03)} & 0.18 (0.04) \\
& 
    & 1 & 0.38 (0.01) & \textbf{0.38 (0.02)} & 0.39 (0.01) \\
\midrule
\multirow{3}{*}{\rotatebox{90}{\textbf{SELFIES \cite{maus2022local}}}}
& \multirow{1}{*}{Transformer (pdop)} 
    & 1 & 0.51 (0.02) & \textbf{0.47 (0.01)} & 0.48 (0.01) \\
\cmidrule{2-6}
& \multirow{1}{*}{Transformer (rano)} 
    & 1 & 0.34 (0.02) & \textbf{0.37 (0.01)} & 0.34 (0.02) \\
\cmidrule{2-6}
& \multirow{1}{*}{Transformer (zale)} 
    & 1 & 0.47 (0.0) & \textbf{0.44 (0.01)} & 0.45 (0.02) \\
\midrule
\multirow{9}{*}{\rotatebox{90}{\textbf{SMILES}}}
& \multirow{3}{*}{GRU} 
    & 0.05 & 2.06 (0.3) & \textbf{1.71 (0.24)} & 2.02 (0.22) \\
& 
    & 0.1 & 2.11 (0.28) & \textbf{1.74 (0.2)} & 1.94 (0.22) \\
& 
    & 1 & 2.4 (0.31) & \textbf{2.1 (0.32)} & 2.34 (0.31) \\
\cmidrule{2-6}
& \multirow{3}{*}{LSTM} 
    & 0.05 & 2.31 (0.21) & \textbf{2.32 (0.19)} & 2.15 (0.18) \\
& 
    & 0.1 & 1.74 (0.5) & \textbf{1.85 (0.34)} & 1.47 (0.49) \\
& 
    & 1 & 2.8 (0.25) & \textbf{3.09 (0.16)} & 2.42 (0.3) \\
\cmidrule{2-6}
& \multirow{3}{*}{Transformer} 
    & 0.05 & 1.82 (0.2) & \textbf{1.79 (0.13)} & 1.74 (0.17) \\
& 
    & 0.1 & 2.19 (0.16) & \textbf{2.24 (0.1)} & 2.09 (0.17) \\
& 
    & 1 & 1.72 (0.18) & \textbf{2.0 (0.14)} & 1.96 (0.14) \\
\bottomrule
\end{tabular}
\end{table}

\section{Additional experimental results}\label{sec:more_exprs}
\begin{table}[H]
\centering
\renewcommand{\arraystretch}{1.1} 
\caption{Average of the top solution found during LSO (higher is better), across datasets and decoder architectures. We {\bf bold} the best method and \underline{underline} the second-best. The average ranking for each method (lower is better) is provided, along with the number of times each method is within one standard deviation of the best. $\score$ and $\prior$ achieve the highest value most frequently (14 out of 30) and outperforms other methods both in terms of the average ranking and the frequency of being within one standard deviation of the best result.\\}\label{tab:top1}
\scriptsize
\begin{tabular}{lccccccccc}
\toprule
 & \textbf{Architecture} & \textbf{$\beta$} &\textbf{\score} & \textbf{\bound} & \textbf{\uc} & \textbf{\noreg} & \textbf{\prior} & \textbf{\turbo} & \textbf{\likelihood} \\
\midrule
\multirow{9}{*}{\rotatebox{90}{\textbf{Expressions}}}
& \multirow{3}{*}{GRU} 
    & 0.05 & -0.55 (0.04) & -0.56 (0.04) & -0.59 (0.04) & -0.45 (0.05) & \textbf{-0.4 (0.07)} & -0.73 (0.12) & \textbf{-0.4 (0.07)} \\
& 
    & 0.1 & -0.45 (0.03) & -0.46 (0.02) & -0.47 (0.05) & \underline{-0.43 (0.02)} & \textbf{-0.37 (0.03)} & -0.61 (0.07) & \underline{-0.43 (0.02)} \\
& 
    & 1 & -0.47 (0.03) & -0.77 (0.02) & -0.51 (0.04) & \underline{-0.46 (0.02)} & -0.47 (0.03) & -0.54 (0.05) & \textbf{-0.43 (0.01)} \\
\cmidrule{2-10}
& \multirow{3}{*}{LSTM} 
    & 0.05 & -0.43 (0.02) & -0.56 (0.04) & -0.52 (0.05) & -0.43 (0.01) & \underline{-0.41 (0.01)} & -0.43 (0.02) & \textbf{-0.4 (0.01)} \\
& 
    & 0.1 & \textbf{-0.32 (0.05)} & -0.44 (0.04) & -0.39 (0.04) & -0.38 (0.02) & -0.4 (0.01) & -0.39 (0.0) & \textbf{-0.32 (0.04)} \\
& 
    & 1 & \textbf{-0.86 (0.0)} & -1.02 (0.07) & \textbf{-0.86 (0.0)} & \textbf{-0.86 (0.0)} & \textbf{-0.86 (0.0)} & \textbf{-0.86 (0.0)} & -0.91 (0.04) \\
\cmidrule{2-10}
& \multirow{3}{*}{Transformer} 
    & 0.05 & -0.43 (0.03) & -0.8 (0.1) & -0.57 (0.05) & -0.44 (0.02) & \textbf{-0.37 (0.05)} & -0.44 (0.02) & \underline{-0.38 (0.04)} \\
& 
    & 0.1 & \underline{-0.36 (0.03)} & -0.65 (0.1) & -0.55 (0.04) & -0.39 (0.01) & \textbf{-0.35 (0.04)} & -0.41 (0.02) & -0.41 (0.02) \\
& 
    & 1 & \textbf{-0.52 (0.05)} & -0.82 (0.1) & \underline{-0.58 (0.05)} & \underline{-0.58 (0.04)} & -0.62 (0.09) & \underline{-0.58 (0.1)} & -0.65 (0.08) \\
\midrule
\multirow{9}{*}{\rotatebox{90}{\textbf{SELFIES}}}
& \multirow{3}{*}{Transformer (pdop)} 
    & 0.05 & \textbf{0.43 (0.0)} & 0.36 (0.02) & 0.42 (0.0) & 0.42 (0.0) & 0.42 (0.0) & 0.15 (0.03) & \textbf{0.43 (0.0)} \\
& 
    & 0.1 & \textbf{0.43 (0.0)} & 0.34 (0.02) & 0.42 (0.01) & 0.42 (0.0) & 0.41 (0.01) & 0.09 (0.03) & \textbf{0.43 (0.01)} \\
& 
    & 1 & \textbf{0.41 (0.01)} & 0.36 (0.02) & 0.39 (0.01) & \underline{0.4 (0.0)} & 0.38 (0.01) & 0.31 (0.03) & \underline{0.4 (0.0)} \\
\cmidrule{2-10}
& \multirow{3}{*}{Transformer (rano)} 
    & 0.05 & \underline{0.33 (0.01)} & 0.07 (0.01) & \textbf{0.36 (0.01)} & \underline{0.33 (0.01)} & 0.32 (0.01) & 0.2 (0.02) & 0.31 (0.01) \\
& 
    & 0.1 & \underline{0.33 (0.01)} & 0.08 (0.02) & \textbf{0.36 (0.02)} & 0.32 (0.01) & 0.32 (0.01) & 0.1 (0.01) & \underline{0.33 (0.01)} \\
& 
    & 1 & 0.31 (0.01) & 0.16 (0.02) & \textbf{0.39 (0.02)} & 0.31 (0.01) & \underline{0.33 (0.02)} & 0.05 (0.02) & \underline{0.33 (0.01)} \\
\cmidrule{2-10}
& \multirow{3}{*}{Transformer (zale)} 
    & 0.05 & \underline{0.43 (0.01)} & 0.27 (0.03) & 0.42 (0.01) & \underline{0.43 (0.01)} & 0.42 (0.0) & 0.15 (0.03) & \textbf{0.44 (0.01)} \\
& 
    & 0.1 & \textbf{0.44 (0.01)} & 0.19 (0.03) & 0.42 (0.01) & 0.42 (0.01) & 0.42 (0.01) & 0.16 (0.01) & \textbf{0.44 (0.01)} \\
& 
    & 1 & \textbf{0.42 (0.01)} & 0.38 (0.02) & 0.39 (0.01) & \textbf{0.42 (0.01)} & 0.37 (0.01) & 0.31 (0.03) & \textbf{0.42 (0.01)} \\
\midrule
\multirow{3}{*}{\rotatebox{90}{\textbf{SELFIES (\cite{maus2022local})}}}
& \multirow{1}{*}{Transformer (pdop)} 
    & 1 & \textbf{0.54 (0.01)} & 0.46 (0.01) & \underline{0.50 (0.01)} & 0.49 (0.01) & 0.36 (0.02) & 0.48 (0.02) & 0.42 (0.00) \\
\cmidrule{2-10}
& \multirow{1}{*}{Transformer (rano)} 
    & 1 & \textbf{0.37 (0.00)} & \textbf{0.37 (0.01)} & \underline{0.36 (0.01)} & \textbf{0.37 (0.01)} & 0.34 (0.01) & 0.35 (0.02) & 0.34 (0.01) \\
\cmidrule{2-10}
& \multirow{1}{*}{Transformer (zale)} 
    & 1 & \textbf{0.47 (0.01)} & 0.43 (0.01) & \underline{0.44 (0.01)} & \textbf{0.47 (0.01)} & 0.39 (0.01) & 0.44 (0.01) & 0.42 (0.01) \\
\midrule
\multirow{9}{*}{\rotatebox{90}{\textbf{SMILES}}}
& \multirow{3}{*}{GRU} 
    & 0.05 & \textbf{3.29 (0.1)} & 1.71 (0.24) & 3.13 (0.07) & \underline{3.26 (0.11)} & 3.18 (0.06) & 2.47 (0.22) & \underline{3.26 (0.08)} \\
& 
    & 0.1 & \textbf{3.55 (0.14)} & 1.74 (0.2) & 3.2 (0.1) & 3.31 (0.16) & 3.15 (0.11) & 2.57 (0.31) & \underline{3.33 (0.12)} \\
& 
    & 1 & 3.85 (0.17) & 2.1 (0.32) & 2.24 (0.28) & 3.66 (0.16) & \textbf{3.89 (0.28)} & 2.48 (0.29) & \textbf{3.89 (0.2)} \\
\cmidrule{2-10}
& \multirow{3}{*}{LSTM} 
    & 0.05 & 3.29 (0.07) & 2.32 (0.19) & 3.12 (0.08) & \underline{3.3 (0.1)} & 3.28 (0.1) & 2.73 (0.33) & \textbf{3.37 (0.09)} \\
& 
    & 0.1 & \textbf{3.66 (0.2)} & 1.85 (0.34) & 2.65 (0.19) & 3.52 (0.22) & \underline{3.57 (0.18)} & 1.78 (0.43) & 3.54 (0.16) \\
& 
    & 1 & \textbf{3.6 (0.14)} & 3.09 (0.16) & 2.6 (0.3) & 3.18 (0.17) & 3.28 (0.17) & 2.71 (0.34) & \underline{3.48 (0.11)} \\
\cmidrule{2-10}
& \multirow{3}{*}{Transformer} 
    & 0.05 & \textbf{3.21 (0.08)} & 1.79 (0.13) & 3.1 (0.07) & 3.14 (0.04) & 3.14 (0.04) & 2.88 (0.24) & \underline{3.19 (0.08)} \\
& 
    & 0.1 & \underline{3.23 (0.04)} & 2.24 (0.1) & \textbf{3.28 (0.08)} & 3.11 (0.05) & 3.09 (0.06) & 2.15 (0.18) & 3.16 (0.06) \\
& 
    & 1 & \textbf{3.2 (0.07)} & 2.0 (0.14) & 2.8 (0.13) & 3.13 (0.07) & 3.11 (0.1) & 2.25 (0.18) & \textbf{3.2 (0.06)} \\
\bottomrule
& \textbf{Average rank}  & & \textbf{2.15} &6.17 & 4.17 & 3.15 & 3.86 & 5.7 & \underline{2.76} \\
& \textbf{\# within 1 std of best}  & & \textbf{19} & 1 & 5 & 8  & 10 & 2 & \underline{16} \\
\bottomrule
\end{tabular}
\end{table}

\begin{table}[htbp]
\centering
\renewcommand{\arraystretch}{1.1} 
\caption{Average of the top 20 solutions found during LSO (higher is better) across datasets and decoder architectures. We {\bf bold} the best method and \underline{underline} the second-best. The average ranking for each method (lower is better) is provided, along with the number of times each method is within one standard deviation of the best. $\score$ achieves the highest value in most experiments (20 out of 30) and outperforms other methods in terms of both the average ranking and the frequency of being within one standard deviation of the best result.\\}\label{tab:top20}
\scriptsize
\begin{tabular}{lccccccccc}
\toprule
 & \textbf{Architecture} & \textbf{$\beta$} &\textbf{\score} & \textbf{\bound} & \textbf{\uc} & \textbf{\noreg} & \textbf{\prior} & \textbf{\turbo} & \textbf{\likelihood} \\
\midrule
\multirow{9}{*}{\rotatebox{90}{\textbf{Expressions}}}
& \multirow{3}{*}{GRU} 
    & 0.05 & -1.43 (0.05) & -1.72 (0.07) & -1.5 (0.05) & \underline{-1.25 (0.06)} & -1.28 (0.05) & -2.1 (0.11) & \textbf{-1.15 (0.09)} \\
& 
    & 0.1 & -1.34 (0.08) & -1.93 (0.09) & -2.01 (0.12) & \underline{-1.21 (0.09)} & \textbf{-1.18 (0.11)} & -2.14 (0.18) & -1.31 (0.08) \\
& 
    & 1 & \textbf{-0.84 (0.02)} & -1.97 (0.07) & -0.91 (0.03) & \textbf{-0.84 (0.02)} & -0.89 (0.05) & -1.34 (0.05) & -0.88 (0.03) \\
\cmidrule{2-10}
& \multirow{3}{*}{LSTM} 
    & 0.05 & -1.06 (0.06) & -1.78 (0.09) & -1.53 (0.06) & \textbf{-0.98 (0.05)} & -1.02 (0.03) & -1.59 (0.13) & \underline{-1.0 (0.07)} \\
& 
    & 0.1 & -0.8 (0.03) & -1.39 (0.08) & -1.09 (0.04) & \underline{-0.79 (0.03)} & \underline{-0.79 (0.03)} & -1.32 (0.13) & \textbf{-0.75 (0.03)} \\
& 
    & 1 & \textbf{-1.79 (0.02)} & -2.04 (0.04) & -2.02 (0.03) & \underline{-1.81 (0.02)} & -1.83 (0.02) & -2.22 (0.08) & \underline{-1.81 (0.03)} \\
\cmidrule{2-10}
& \multirow{3}{*}{Transformer} 
    & 0.05 & -1.0 (0.04) & -2.93 (0.13) & -1.64 (0.08) & -1.03 (0.04) & \underline{-0.99 (0.05)} & -2.16 (0.09) & \textbf{-0.93 (0.05)} \\
& 
    & 0.1 & \textbf{-0.77 (0.02)} & -2.69 (0.25) & -1.79 (0.08) & \textbf{-0.77 (0.04)} & -0.78 (0.03) & -1.39 (0.12) & -0.81 (0.03) \\
& 
    & 1 & \textbf{-1.36 (0.09)} & -2.41 (0.09) & -1.52 (0.09) & -1.5 (0.07) & \underline{-1.41 (0.08)} & -1.77 (0.12) & -1.45 (0.12) \\
\midrule
\multirow{9}{*}{\rotatebox{90}{\textbf{SELFIES}}}
& \multirow{3}{*}{Transformer (pdop)} 
    & 0.05 & \textbf{0.37 (0.0)} & 0.21 (0.01) & 0.36 (0.0) & \textbf{0.37 (0.0)} & \textbf{0.37 (0.0)} & 0.04 (0.01) & \textbf{0.37 (0.0)} \\
& 
    & 0.1 & \textbf{0.36 (0.0)} & 0.19 (0.01) & 0.35 (0.0) & \textbf{0.36 (0.0)} & 0.34 (0.0) & 0.01 (0.0) & \textbf{0.36 (0.0)} \\
& 
    & 1 & \textbf{0.35 (0.0)} & 0.25 (0.01) & 0.31 (0.01) & 0.33 (0.0) & 0.28 (0.01) & 0.23 (0.02) & \underline{0.34 (0.0)} \\
\cmidrule{2-10}
& \multirow{3}{*}{Transformer (rano)} 
    & 0.05 & \underline{0.21 (0.0)} & 0.03 (0.0) & \textbf{0.22 (0.0)} & \underline{0.21 (0.0)} & \underline{0.21 (0.0)} & 0.06 (0.0) & \underline{0.21 (0.0)} \\
& 
    & 0.1 & 0.21 (0.0) & -- & \textbf{0.23 (0.01)} & 0.21 (0.0) & \underline{0.22 (0.0)} & 0.03 (0.01) & \underline{0.22 (0.0)} \\
& 
    & 1 & 0.21 (0.0) & 0.07 (0.0) & \textbf{0.22 (0.01)} & 0.19 (0.0) & 0.2 (0.0) & -- & 0.21 (0.0) \\
\cmidrule{2-10}
& \multirow{3}{*}{Transformer (zale)} 
    & 0.05 & \textbf{0.33 (0.0)} & 0.13 (0.01) & 0.32 (0.0) & \textbf{0.33 (0.0)} & 0.32 (0.0) & 0.04 (0.01) & \textbf{0.33 (0.0)} \\
& 
    & 0.1 & \textbf{0.34 (0.0)} & 0.06 (0.0) & 0.31 (0.01) & 0.32 (0.01) & 0.31 (0.0) & 0.04 (0.01) & \textbf{0.34 (0.0)} \\
& 
    & 1 & \textbf{0.31 (0.0)} & 0.18 (0.02) & 0.26 (0.01) & 0.29 (0.01) & 0.23 (0.01) & 0.17 (0.02) & \underline{0.3 (0.01)} \\
\midrule
\multirow{3}{*}{\rotatebox{90}{\textbf{SELFIES (\cite{maus2022local})}}}
& \multirow{1}{*}{Transformer (pdop)} 
    & 1 & \textbf{0.48 (0.01)} & 0.42 (0.01) &\underline{0.44 (0.00)} &\underline{0.44 (0.00)} & 0.3 (0.01) & \underline{0.44 (0.01)} & 0.35 (0.01) \\
\cmidrule{2-10}
& \multirow{1}{*}{Transformer (rano)} 
    & 1 & \textbf{0.29 (0.00)} & 0.17 (0.03) & 0.27 (0.01) & 0.27 (0.01) & \textbf{0.25 (0.01)} & 0.22 (0.03) & 0.22 (0.01) \\
\cmidrule{2-10}
& \multirow{1}{*}{Transformer (zale)} 
    & 1 & \textbf{0.4 (0.00)} & 0.37 (0.01) & 0.38 (0.01) & \underline{0.39 (0.0)} & 0.27 (0.01) & 0.37 (0.01) & 0.31 (0.01) \\
\midrule
\multirow{9}{*}{\rotatebox{90}{\textbf{SMILES}}}
& \multirow{3}{*}{GRU} 
    & 0.05 & \textbf{2.31 (0.04)} & 0.52 (0.12) & 2.18 (0.05) & 2.21 (0.05) & 2.2 (0.04) & 0.82 (0.15) & \textbf{2.31 (0.03)} \\
& 
    & 0.1 & \textbf{2.3 (0.05)} & 0.12 (0.14) & 1.68 (0.04) & 2.09 (0.07) & 1.93 (0.06) & 0.41 (0.19) & \underline{2.12 (0.07)} \\
& 
    & 1 & \textbf{1.64 (0.14)} & -- & -- & \underline{1.19 (0.2)} & 0.58 (0.28) & -0.16 (0.0) & 1.49 (0.14) \\
\cmidrule{2-10}
& \multirow{3}{*}{LSTM} 
    & 0.05 & \textbf{2.33 (0.03)} & 0.66 (0.12) & 2.02 (0.03) & 2.22 (0.05) & 2.13 (0.05) & 1.54 (0.41) & \underline{2.28 (0.03)} \\
& 
    & 0.1 & \textbf{1.57 (0.1)} & -- & -- & 0.8 (0.12) & 0.9 (0.09) & -- & 1.43 (0.12) \\
& 
    & 1 & \textbf{1.94 (0.1)} & 0.7 (0.2) & 0.95 (0.24) & 1.14 (0.21) & 0.81 (0.26) & 0.52 (0.26) & \underline{1.67 (0.13)} \\
\cmidrule{2-10}
& \multirow{3}{*}{Transformer} 
    & 0.05 & \textbf{2.26 (0.04)} & 0.42 (0.15) & 2.04 (0.05) & 2.22 (0.03) & 2.24 (0.03) & 0.97 (0.19) & \underline{2.25 (0.03)} \\
& 
    & 0.1 & \textbf{2.26 (0.03)} & 0.61 (0.15) & 2.08 (0.04) & 2.21 (0.03) & 2.17 (0.03) & 0.62 (0.14) & \underline{2.23 (0.02)} \\
& 
    & 1 & \textbf{2.17 (0.05)} & 0.23 (0.12) & 1.32 (0.09) & 1.98 (0.06) & 1.82 (0.06) & 0.48 (0.18) & \underline{2.16 (0.05)} \\
\bottomrule
& \textbf{Average rank}  & &   \textbf{1.9} & 6.4 & 4.3 & 2.87 & 3.8 & 6.1 & \underline{2.53} \\
& \textbf{\# within 1 std of best}  & & \textbf{22} & 0 & 3 & 7  & 5 & 0 & \underline{13} \\
\bottomrule
\end{tabular}
\end{table}

\begin{figure}
    \centering
    \includegraphics[width=0.6\linewidth]{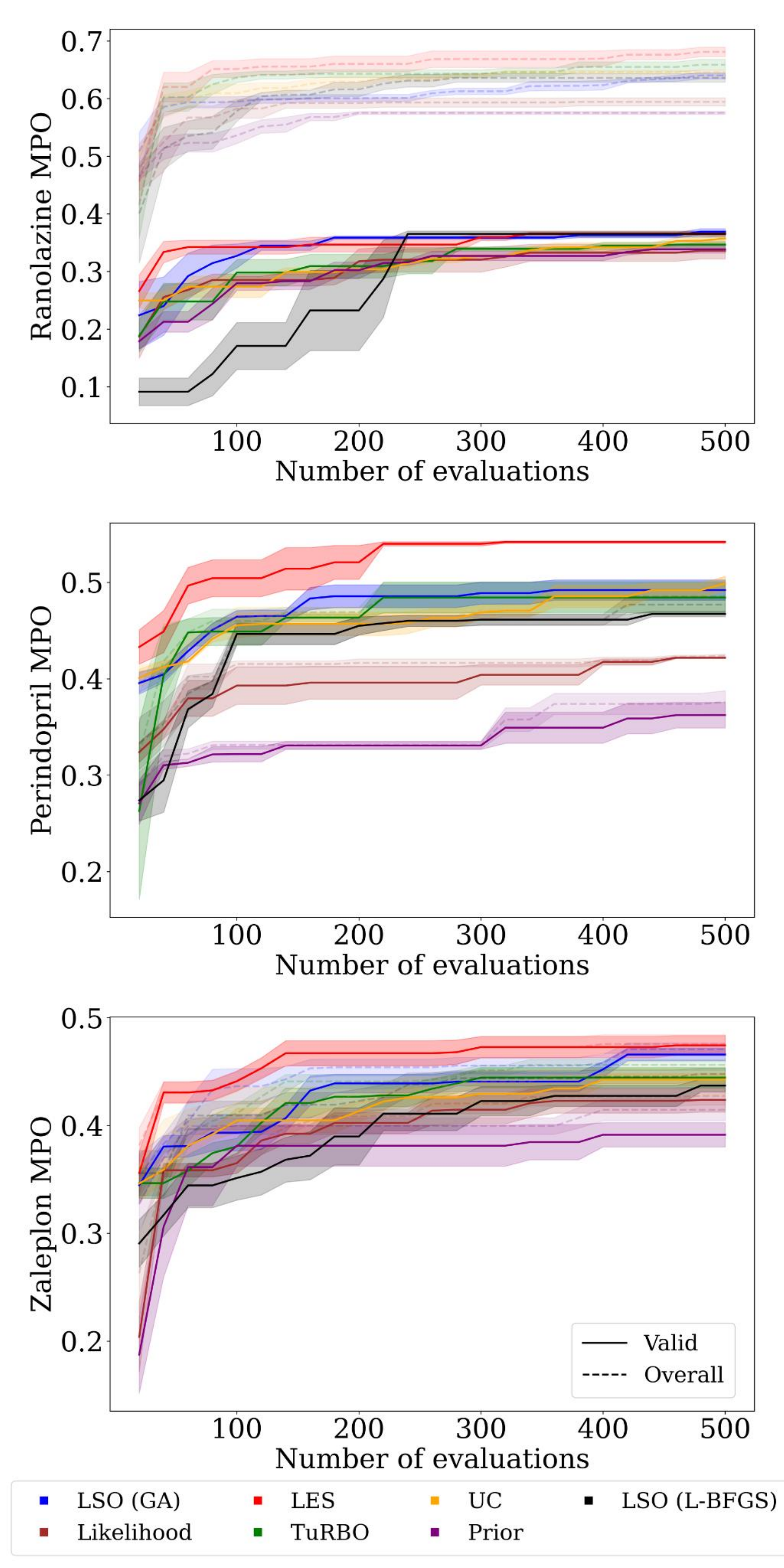}
    \caption{Cumulative true objective for the best solution found during Bayesian optimization with the pre-trained SELFIES-VAE (\cite{maus2022local}). Each method is shown in a distinct color. Solid lines represent solutions passing quality filters, while dashed lines include all evaluations. $\score$ achieves the best performance on $\pdop$ and is comptetitive on $\zale$ and $\rano$.}
    \label{fig:cumobj1}
\end{figure}

\begin{table}[htbp]
\caption{AUROC values (higher is better) for identifying valid data points within the latent space, across datasets and decoder architectures. Data points are sampled from the training data, the VAE prior (standard Gaussian) and out of distribution (Gaussian with std of 5). $\score$ achieves the best performance in most cases (18 out of 22) and on average. In addition, $\score$ achieves AUROC values of at least 0.75 in all cases, indicating it can effectively differentiate valid from invalid data points.\\}\label{tab:roc_scores}
\centering
\renewcommand{\arraystretch}{1.1} 
\small
\begin{tabular}{llclrrrrr}
\toprule
Dataset & Arch. & $\beta$ & $\score$\;  & Polarity & Prior & Uncertainty & Train distances & Likelihood \\
\midrule
\multirow{9}{*}{SMILES} 
& \multirow{3}{*}{GRU} 
    & 0.05 & 0.93 & 0.42 & 0.09 & 0.85 & 0.85 & \textbf{0.94} \\
& 
    & 0.1  & \textbf{0.94} & 0.72 & 0.12 & 0.84 & 0.86 & \textbf{0.94} \\
& 
    & 1.0  & 0.91 & 0.35 & 0.16 & 0.87 & 0.80& \textbf{0.92} \\ \cline{2-9}
& \multirow{3}{*}{LSTM} 
    & 0.05 & \textbf{0.99} & 0.93 & 0.07 & \textbf{0.99} & 0.91 & 0.98 \\
& 
    & 0.1  & 0.89 & 0.67 & 0.21 & 0.89 & 0.81 & \textbf{0.9} \\
& 
    & 1.0  & \textbf{0.97} & 0.76 & 0.12 & 0.95 & 0.85 & 0.96 \\ \cline{2-9}
& \multirow{3}{*}{Transformer} 
    & 0.05 & \textbf{0.93} & 0.89 & 0.14 & 0.84 & 0.77 & 0.92 \\
& 
    & 0.1  & \textbf{0.94} & 0.91 & 0.14 & 0.87 & 0.86 & 0.93 \\
& 
    & 1.0  & \textbf{0.97} & 0.93 & 0.10 & 0.89 & 0.85 & 0.95 \\ \midrule

\multirow{9}{*}{Expressions} 
& \multirow{3}{*}{GRU} 
    & 0.05 & \textbf{0.96} & 0.89 & 0.38 & \textbf{0.96} & 0.67 & 0.88 \\
& 
    & 0.1  & \textbf{0.94} & 0.86 & 0.42 & \textbf{0.94} & 0.71 & 0.8 \\
& 
    & 1.0  & \textbf{0.94} & 0.80 & 0.57 & \textbf{0.94} & 0.75 & 0.89 \\ \cline{2-9}
& \multirow{3}{*}{LSTM} 
    & 0.05 & \textbf{0.96} & 0.86 & 0.38 & 0.90 & 0.67 & \textbf{0.96} \\
& 
    & 0.1  & \textbf{0.95} & 0.83 & 0.37 & 0.91 & 0.66 & \textbf{0.95} \\
& 
    & 1.0  & \textbf{0.95} & 0.79 & 0.56 & 0.91 & 0.72 & 0.91 \\ \cline{2-9}
& \multirow{3}{*}{Transformer} 
    & 0.05 & \textbf{0.91} & 0.79 & 0.43 & 0.86 & 0.71 & 0.90 \\
& 
    & 0.1  & \textbf{0.91} & 0.79 & 0.53 & 0.87 & 0.78 & 0.89 \\
& 
    & 1.0  & 0.86 & 0.61 & 0.70 & \textbf{0.92} & 0.89 & \textbf{0.92} \\ \midrule

\multirow{3}{*}{SELFIES} 
& \multirow{3}{*}{Transformer} 
    & 0.05 & \textbf{1.0} & 0.99 & 0.02 & 0.99 & 0.62 & 0.97 \\
& 
    & 0.1  & \textbf{0.99} & 0.98 & 0.03 & 0.96 & 0.81 & 0.98 \\
& 
    & 1.0  & \textbf{0.95} & 0.94 & 0.06 & 0.85 & 0.72 & 0.93 \\ \midrule

SELFIES~ (\cite{maus2022local})
& Transformer 
    & --   & \textbf{0.75} & 0.39 & 0.69 & 0.33 & 0.69 & 0.70 \\

\midrule
& Average  & & \textbf{0.93} & 0.78 & 0.29 & 0.88 & 0.77 & 0.91 \\
\bottomrule
\end{tabular}

\end{table}

\begin{table}
\centering
\renewcommand{\arraystretch}{1.2} 
\caption{Effect of increasing $\lambda$ parameter for $\score$, for the expressions dataset. Increasing the value of the parameter $\lambda$ increases the percentage of valid solution in all cases.}\label{tab:exp_lambda}
\small
\begin{tabular}{llccc}
\toprule
\textbf{Architecture} & \textbf{$\beta$} & \textbf{\score} ($\lambda=0.05$) & \textbf{\score} ($\lambda=0.5$) \\
\midrule

\multirow{3}{*}{GRU} 
        & \textbf{0.05} 
            & 0.92 
            & 0.96 \\
        & \textbf{0.1} 
            & 0.7 
            & 0.82 \\
        & \textbf{1} 
            & 0.72 
            & 0.87 \\
\midrule

\multirow{3}{*}{LSTM} 
        & \textbf{0.05} 
            & 0.93 
            & 0.98 \\
        & \textbf{0.1} 
            & 0.93 
            & 0.98 \\
        & \textbf{1} 
            & 0.76 
            & 0.97 \\
\midrule

\multirow{3}{*}{Transformer} 
        & \textbf{0.05} 
            & 0.87 
            & 0.94 \\
        & \textbf{0.1} 
            & 0.88 
            & 0.96 \\
        & \textbf{1} 
            & 0.83 
            & 0.85 \\
\bottomrule
\end{tabular}
\end{table}

\begin{table}[htbp]
\centering
\renewcommand{\arraystretch}{1.2} 
\caption{Proportion of valid solutions found during LSO (higher is better) across datasets and decoder architectures.We bold the best method (higher is better) and underline the second best.
$\score$ improves the validity of the solutions compared with $\noreg$ (which is $\score$ with $\lambda=0$) across all datasets.\\} \label{tab:validity}
\small
\begin{tabular}{lccccccccc}
\toprule
 & \textbf{Architecture} & \textbf{$\beta$} &\textbf{\score} & \textbf{\bound} & \textbf{\uc} & \textbf{\noreg} & \textbf{\prior} & \textbf{\turbo} & \textbf{\likelihood} \\
\midrule
\multirow{9}{*}{\rotatebox{90}{\textbf{Expressions}}}
& \multirow{3}{*}{GRU} 
    & 0.05 & 0.92 & 0.59 & \textbf{1.0} & 0.91 & 0.91 & \underline{0.94} & 0.89 \\
& 
    & 0.1 & 0.7 & 0.57 & \textbf{1.0} & 0.64 & 0.66 & \underline{0.9} & 0.67 \\
& 
    & 1 & 0.72 & 0.45 & \textbf{0.99} & 0.69 & 0.69 & \underline{0.83} & 0.69 \\
\cmidrule{2-10}
& \multirow{3}{*}{LSTM} 
    & 0.05 & 0.93 & 0.62 & \textbf{1.0} & 0.88 & 0.89 & \underline{0.94} & 0.92 \\
& 
    & 0.1 & 0.93 & 0.67 & \textbf{1.0} & 0.9 & 0.89 & \underline{0.94} & 0.92 \\
& 
    & 1 & 0.76 & 0.58 & \textbf{0.99} & 0.65 & 0.66 & \underline{0.89} & 0.73 \\
\cmidrule{2-10}
& \multirow{3}{*}{Transformer} 
    & 0.05 & 0.87 & 0.37 & \textbf{1.0} & 0.83 & 0.85 & \underline{0.9} & 0.84 \\
& 
    & 0.1 & 0.88 & 0.28 & \textbf{1.0} & 0.8 & 0.81 & \underline{0.89} & 0.85 \\
& 
    & 1 & 0.83 & 0.36 & \textbf{1.0} & 0.74 & 0.76 & \underline{0.87} & 0.79 \\
\midrule
\multirow{9}{*}{\rotatebox{90}{\textbf{SELFIES}}}
& \multirow{3}{*}{Transformer (pdop)} 
    & 0.05 & \underline{0.8} & 0.05 & \textbf{0.81} & 0.76 & 0.73 & 0.14 & 0.77 \\
& 
    & 0.1 & \underline{0.68} & 0.04 & \textbf{0.71} & 0.62 & 0.56 & 0.14 & 0.64 \\
& 
    & 1 & \textbf{0.59} & 0.08 & \underline{0.55} & 0.49 & 0.44 & 0.08 & 0.53 \\
\cmidrule{2-10}
& \multirow{3}{*}{Transformer (rano)} 
    & 0.05 & \underline{0.66} & 0.02 & \textbf{0.73} & 0.61 & 0.57 & 0.09 & 0.62 \\
& 
    & 0.1 & \underline{0.57} & 0.01 & \textbf{0.67} & 0.52 & 0.44 & 0.04 & 0.53 \\
& 
    & 1 & \underline{0.43} & 0.02 & \textbf{0.49} & 0.36 & 0.28 & 0.01 & 0.39 \\
\cmidrule{2-10}
& \multirow{3}{*}{Transformer (zale)} 
    & 0.05 & \underline{0.71} & 0.08 & \textbf{0.74} & 0.66 & 0.62 & 0.12 & 0.69 \\
& 
    & 0.1 & \underline{0.66} & 0.01 & \textbf{0.67} & 0.59 & 0.51 & 0.08 & 0.63 \\
& 
    & 1 & \textbf{0.54} & 0.18 & \underline{0.47} & 0.43 & 0.35 & 0.17 & 0.46 \\
\midrule
\multirow{3}{*}{\rotatebox{90}{\textbf{SELFIES (\cite{maus2022local})}}}
& \multirow{1}{*}{Transformer (pdop)} 
    & 1 & \underline{0.69} & 0.45 & 0.56 & 0.58 & \textbf{0.75} & 0.43 & 0.68 \\
\cmidrule{2-10}
& \multirow{1}{*}{Transformer (rano)} 
    & 1 & \textbf{0.26} & 0.07 & 0.14 & 0.19 & \underline{0.20} & 0.11 & 0.15 \\
\cmidrule{2-10}
& \multirow{1}{*}{Transformer (zale)} 
    & 1 & 0.65 & 0.52 &  0.66 & 0.66 & \underline{0.75} & 0.48 & \textbf{0.76} \\
\midrule
\multirow{9}{*}{\rotatebox{90}{\textbf{SMILES}}}
& \multirow{3}{*}{GRU} 
    & 0.05 & \textbf{0.61} & 0.14 & 0.48 & 0.47 & 0.44 & 0.12 & \underline{0.59} \\
& 
    & 0.1 & \textbf{0.37} & 0.07 & 0.25 & 0.23 & 0.22 & 0.06 & \underline{0.35} \\
& 
    & 1 & \textbf{0.09} & 0.02 & 0.05 & 0.05 & 0.04 & 0.02 & \underline{0.08} \\
\cmidrule{2-10}
& \multirow{3}{*}{LSTM} 
    & 0.05 & \textbf{0.6} & 0.11 & 0.44 & 0.42 & 0.39 & 0.11 & \underline{0.57} \\
& 
    & 0.1 & \textbf{0.08} & 0.01 & 0.06 & 0.05 & 0.04 & 0.01 & \underline{0.07} \\
& 
    & 1 & \textbf{0.16} & 0.09 & 0.07 & 0.08 & 0.07 & 0.09 & \underline{0.12} \\
\cmidrule{2-10}
& \multirow{3}{*}{Transformer} 
    & 0.05 & \underline{0.7} & 0.42 & \underline{0.7} & 0.68 & 0.68 & 0.31 & \textbf{0.72} \\
& 
    & 0.1 & \underline{0.65} & \textbf{0.67} & 0.63 & 0.6 & 0.55 & 0.41 & 0.64 \\
& 
    & 1 & \textbf{0.48} & 0.35 & \underline{0.46} & 0.34 & 0.26 & 0.3 & 0.45 \\
\bottomrule
& \textbf{Average}  & & \underline{0.62} & 0.26 & \textbf{0.64} & 0.55 & 0.53 & 0.38 & 0.59 \\

\bottomrule
\end{tabular}
\end{table}

\end{document}